%% file: main.tex
\documentclass{article}
\usepackage{arxiv}

\usepackage[utf8]{inputenc} 
\usepackage[T1]{fontenc}    
\usepackage{hyperref}       
\usepackage{url}            
\usepackage{booktabs}       
\usepackage{amsfonts}       
\usepackage{nicefrac}       
\usepackage{microtype}      
\usepackage{lipsum}		
\usepackage{graphicx}
\usepackage{natbib}
\usepackage{doi}
\input{math_commands}

\usepackage{hyperref}
\usepackage{url}
\usepackage{amsthm}
\usepackage{amssymb}
\newtheorem{theorem}{Theorem}

\newtheorem{lemma}{Lemma}
\newtheorem{definition}{Definition}
\newtheorem{assumption}{Assumption}
\newtheorem{remark}{Remark}
\newtheorem{proposition}{Proposition}
\usepackage{microtype}
\usepackage{graphicx}
\usepackage{graphics}
\usepackage[T1]{fontenc}
\usepackage{float} 
\usepackage{changepage}
\usepackage{makecell}
\usepackage{multirow}
\usepackage{array}
\usepackage{wrapfig}
\usepackage[colorinlistoftodos]{todonotes}
\usepackage{tikz}

\title{Exponential escape efficiency of SGD from sharp minima in non-stationary regime}

\date{}

\author{ Hikaru Ibayashi\\
	University of Southern California\\
	\texttt{ibayashi@usc.edu}
	\And
	Masaaki Imaizumi\\
	The University of Tokyo / RIKEN Center for AIP\\
	\texttt{imaizumi@g.ecc.u-tokyo.ac.jp}
}

\renewcommand{\headeright}{Preprint under review}
\renewcommand{\undertitle}{}
\renewcommand{\shorttitle}{Exponential escape efficiency of SGD}

\hypersetup{
pdftitle={Exponential escape efficiency of SGD from sharp minima in non-stationary regime},
pdfauthor={Hikaru Ibayashi, Masaaki Imaizumi},
pdfkeywords={deep learning, learning dynamics, SGD, flat minima},
}

\begin{document}
\maketitle

\begin{abstract}
We show that stochastic gradient descent (SGD) escapes from sharp minima exponentially fast even before SGD reaches stationary distribution.
SGD has been a de-facto standard training algorithm for various machine learning tasks.
However, there still exists an open question as to why SGDs find highly generalizable parameters from non-convex target functions,  such as the loss function of neural networks.
An ``escape efficiency'' has been an attractive notion to tackle this question, which measures how SGD efficiently escapes from sharp minima with potentially low generalization performance.
Despite its importance, the notion has the limitation that it works only when SGD reaches a stationary distribution after sufficient updates.
In this paper, we develop a new theory to investigate escape efficiency of SGD with Gaussian noise, by introducing the Large Deviation Theory for dynamical systems.
Based on the theory, we prove that the fast escape form sharp minima, named exponential escape, occurs in a non-stationary setting, and that it holds not only for continuous SGD but also for discrete SGD.
A key notion for the result is a quantity called ``steepness,'' which describes the SGD's stochastic behavior throughout its training process.
Our experiments are consistent with our theory. 
\end{abstract}

\keywords{Deep learning \and Stochastic gradient descent \and Flat minima}

\input{introduction}
\input{setting}
\input{quasi-potential_theory}
\input{main_results}
\input{experiment}
\input{comparison}
\input{related_works}
\bibliographystyle{unsrtnat}
\bibliography{main}
\appendix
\input{stability_assumptions.tex}
\input{mean_exit_time}
\input{hamilton_jacobi_quation}
\end{document}

%% file: math_commands.tex
\usepackage{amsmath,amsfonts,bm}

\def\eqref#1{equation~\ref{#1}}

\def\1{\bm{1}}

\DeclareMathAlphabet{\mathsfit}{\encodingdefault}{\sfdefault}{m}{sl}
\SetMathAlphabet{\mathsfit}{bold}{\encodingdefault}{\sfdefault}{bx}{n}

\newcommand{\R}{\mathbb{R}}



%% file: introduction.tex
\section{Introduction}
Stochastic gradient descent (SGD) has become the de facto standard optimizer
in modern machine learning, especially deep learning.
However, its prevalence has opened up a theoretical question:
why can SGD find generalizable solutions in complicated models
such as neural networks?
The loss landscapes of neural networks are known to be highly non-convex \citep{li2018visualizing},
difficult to minimize \citep{blum1992training}, and full of non-generalizable minima \citep{Zhang2016-ed}.
It is an important task to answer this open question
to attain a solid understanding of modern machine learning.

In recent years,
``escape efficiency from sharp minima''
has emerged as a promising narrative for the SGD's generalization.
``Sharp minima'' mean the model parameters as local minima of loss functions and that are sensitive 
to perturbations.
They are known to deteriorate generalization ability by several empirical and theoretical studies
\citep{Keskar2016-tn, Dziugaite2017-qt, Jiang2019-ci}.
The ``escape efficiency,'' the counterpart of the narrative,
is a measure for how fast SGD moves out of the neighborhood of minima.
In a few words, the narrative claims
that SGD can find generalizable minima
because it has high escape efficiency from sharp minima \citep{Zhu2019-og, Xie2020-ty}.
This narrative is aligned with actual phenomena.
The left panel of Fig. \ref{fig:fluctuation} shows how SGD updates affect sharpness of parameters throughout the training of a neural network,
where the sharpness oscillates widely in the early phase of the training,
and then becomes smaller toward the end.
This suggests that SGD repeatedly jumps out of sharp minima
and eventually reaches flat minima (Fig. \ref{fig:fluctuation}, right).
\begin{figure}[ht]
  \centering
  \begin{minipage}[c]{0.49\textwidth}
    \vspace{0.75cm}
  \centering
      \includegraphics[width=\textwidth]{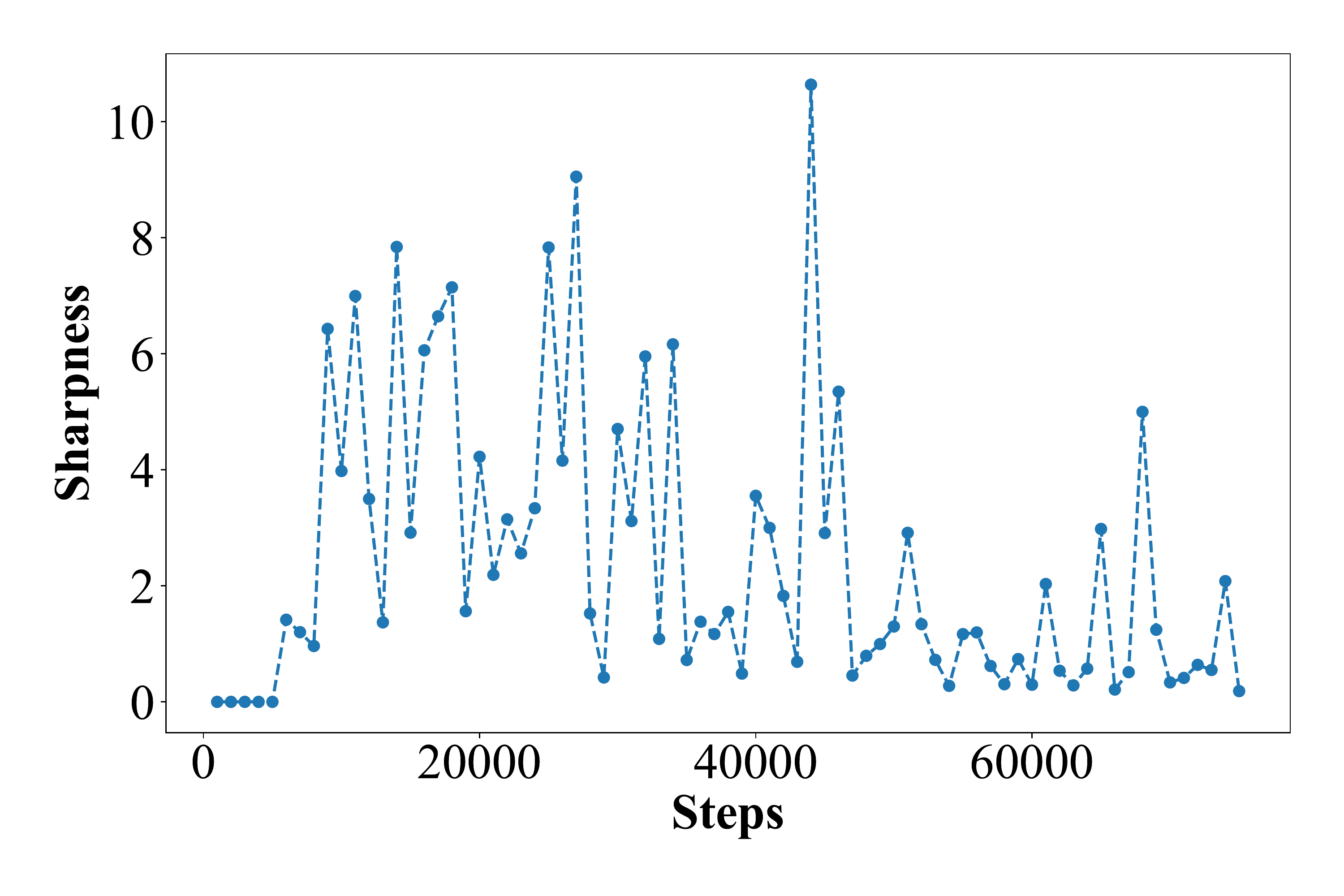}
  \end{minipage}
  \begin{minipage}[c]{0.49\textwidth}
  \centering
      \includegraphics[width=\textwidth]{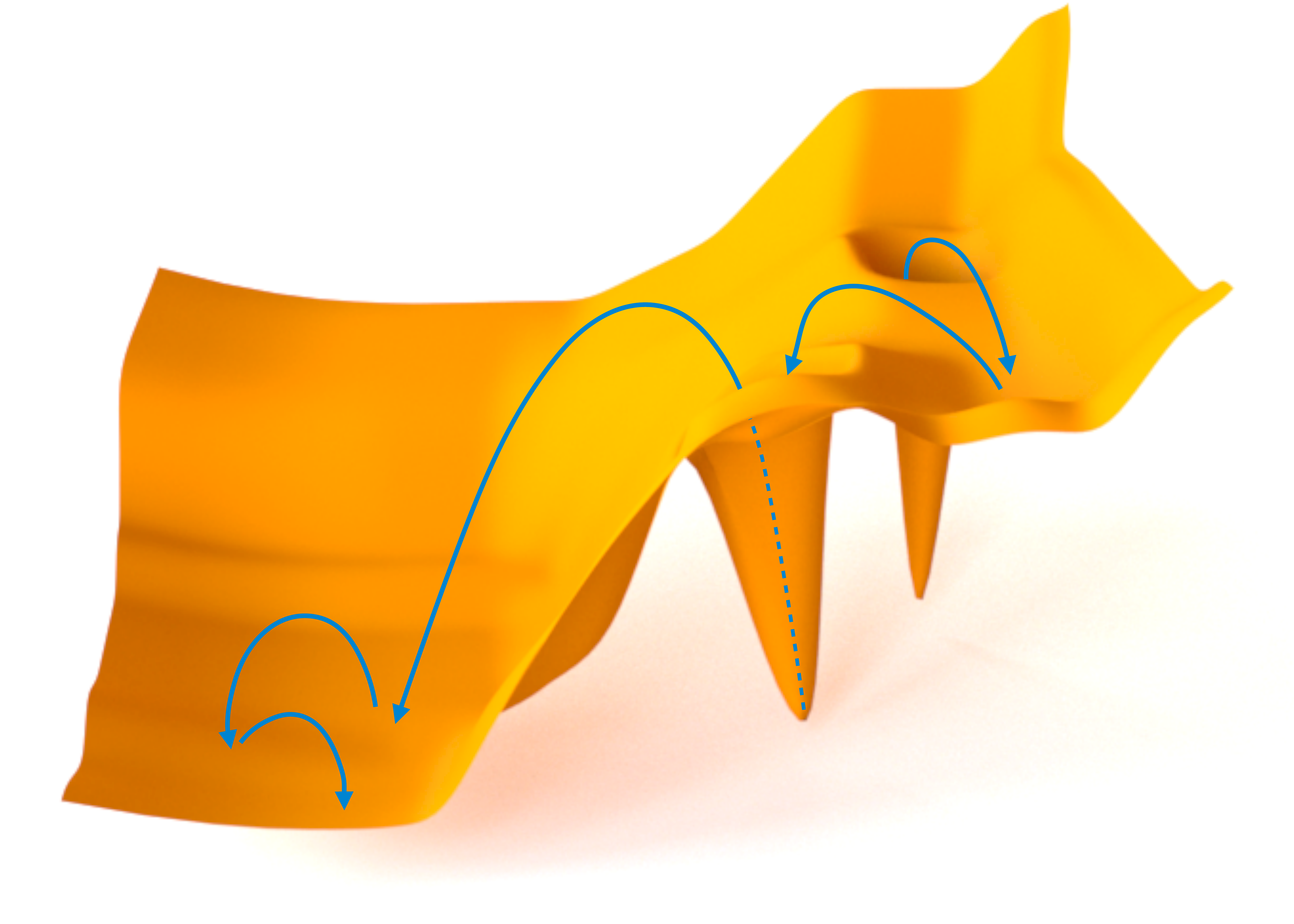}
  \end{minipage}
  \caption{Sharpness dynamics throughout a training via SGD (left), where we used the sharpness defined in \citep[Metric 2.1]{Keskar2016-tn}.
  As is shown, sharpness fluctuates during the training and becomes small toward the end of training.
  This means SGD jumps out from sharp minima to find flat and generalizable minima (right).
  We used VGG \citep{simonyan2014very}
  fed with CIFAR-10 \citep{krizhevsky2009learning} with cross entropy loss.}
  \label{fig:fluctuation}
\end{figure}

Many theoretical studies have investigated the escape efficiency to quantify SGD's escaping behavior.
Regarding SGD as a gradient descent with noise,
recent studies have identified that the noise part plays a key role in escaping.
It was shown that high escape efficiency is realized
by the so-called ``anisotropic noise'' of SGD, which means the noise  
with the various magnitudes among directions \citep{Zhu2019-og}.
\cite{Jastrzebski2017-gm} formulated
the effect of anisotropic noise in the stationary regime,
where SGD has reached a stationary distribution after many iterations.
By limiting the target to the stationary regime,
they took advantage of the theoretical analogy between SGD and thermodynamics.
\cite{Xie2020-ty} further elaborated this approach
and found that escape efficiency can be viewed as Kramer’s escape rate,
which is a well-used formula in various fields of science \citep{Kramers1940-fn}.
Their result revealed that SGD has
the \textbf{``exponential escape efficiency''} from sharp minima, which means the sharpness exponentially increases SGD's escape efficiency.

With these progressive refinements,
a remaining task is how to go beyond the stationary regime.
Although physics and chemistry commonly assume
that a system has reached a stationary distribution \citep{eyring1935activated, hanggi1986escape},
the stationary regime does not apply to the analysis of SGD
due to the following reasons.
First, it is shown that
SGD forms a stationary distribution only on the very limited objective functions \citep{dieuleveut2017bridging,chen2021stationary}.
Secondly, even when such a stationary distribution exists,
SGD takes $O(d)$ steps to reach it,
where $d$ is a number of parameters of a model \citep{Raginsky2017-re}.
Since common neural networks have numerous parameters, the stationary regime may not be directly applicable to the actual SGD dynamics.

In this paper,
we propose a novel formulation of exponential escape efficiency of SGD with Gaussian noise in the non-stationary regime,
by introducing the Large Deviation Theory \citep{Freidlin2012-iz, Dembo2010-fy},
a fundamental theory of stochastic systems.
We formulate the escape efficiency of SGD through a notion of ``steepness,'' which intuitively means the hardness to climb up loss surface along a given trajectory (Fig. \ref{fig:steepness}).
Large Deviation Theory provides that
the escape efficiency is described by steepness of a trajectory starting from the minimum:
\begin{align*}
\textrm{Escape efficiency} \sim \exp{\left[-V_0\right]}\quad(V_0: \textrm{steepness from a minimum})
\end{align*}
Based on this analysis, we show the following main result on the escape efficiency of SGD:
\begin{align*}
\textrm{Theorem \ref{thm:result} (informal):   Escape efficiency of SGD}
  &\sim \exp \left[ -\frac{B}{\eta} \Delta L \lambda^{-\frac{1}{2}}_\mathrm{max}\right],
\end{align*}
where $B$ is a batch size, $\eta$ is a learning rate, $\Delta L$ is the depth of the minimum, and $\lambda_\mathrm{max}$ is the sharpness of a loss function
(Definition \ref{def:sharpness}).
We can see that as sharpness of the minimum increases, i.e. $\lambda_{\max}$ increases, the escape efficiency increases exponentially.
This is the first result showing that SGD has exponential escape efficiency from sharp minima,
even out of the stationary regime.
As a further benefit, our formulation can be easily extended to
the discrete update rule of SGD (Theorem \ref{thm:result2}).
\begin{figure}[t]
  \centering
  \includegraphics[width=0.9\textwidth]{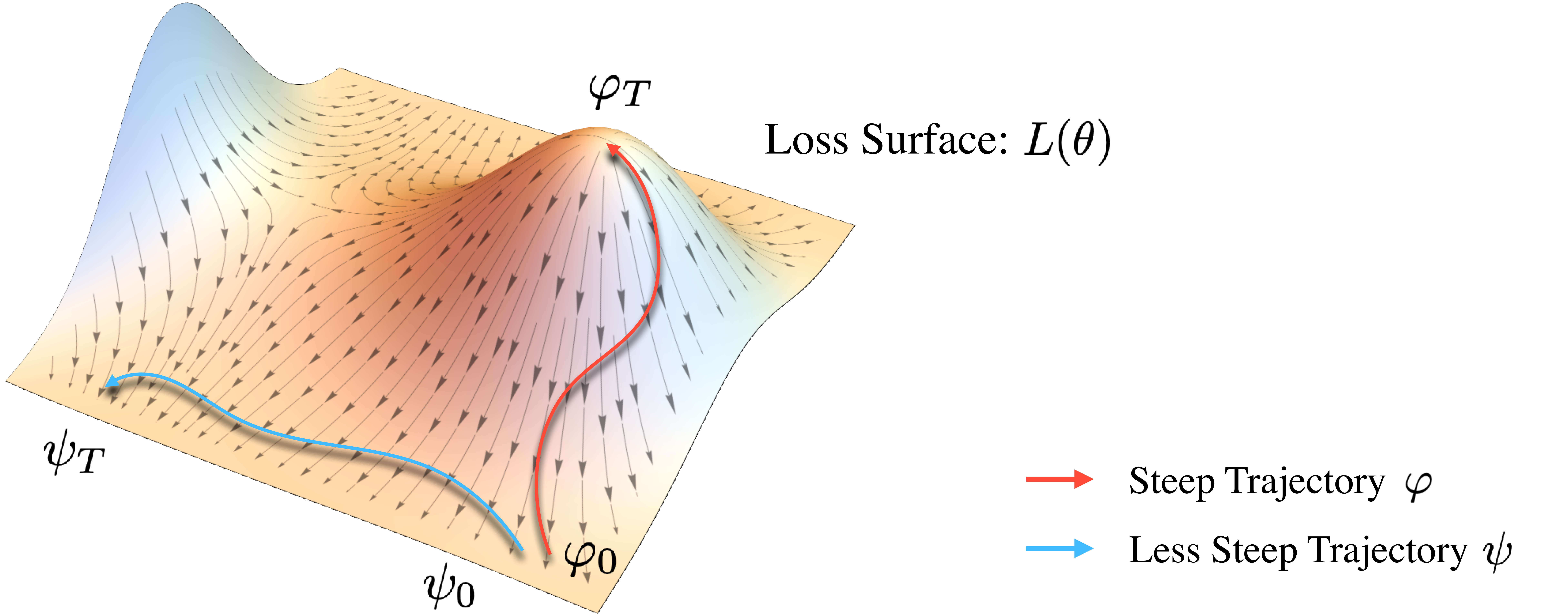}
  \vspace{0.3cm}
  \caption{Visual illustration of steepness (Definition \ref{def:steepness}). The steepness of $\varphi$, $S_{T}(\varphi)$, is greater than $S_{T}(\psi)$ because $\varphi$ moves against the vector field of gradient $-\nabla L(\theta)$.}
  \label{fig:steepness}
\end{figure}

The rest of this paper is organized as follows.
Section \ref{sec:formulation} formulates the SGD's escape problem.
Section \ref{sec:LDT_for_SGD} presents the traditional Large Deviation Theory and its application to the SGD's escape problem.
Section \ref{sec:main_analysis} gives the main results on the escape efficiency.
In Section \ref{sec:experiment}, we confirm that the main results are consistent with numerical experiments.

%% file: setting.tex
\section{Problem Formulation}
\label{sec:formulation}
\textbf{Notations}:
For a $k \times k$ matrix $A$, $\lambda_j(A)$ is the $j$-th largest eigenvalue of $A$.
We especially write $\lambda_{\max}(A) = \lambda_1(A)$ and $\lambda_{\min}(A) = \lambda_k(A)$. 
$\mathcal{O}(\cdot)$ denotes Landau's Big-O notation.
$\|\cdot\|$ denotes the Euclidean norm.
Given a time-dependent function $\theta_t$, $\dot{\theta}_t$ denotes the differentiation of $\theta_t$ with respect to $t$.
$N\left(\mu,  \Sigma\right)$ denotes the multivariate Gaussian distribution with the mean $\mu$, and  the covariance $\Sigma$.

\subsection{Stochastic Gradient Descents}

We consider a learning model parameterized by $\theta\in \R^d$, where $d$ is a number of parameters.
Given training examples $\{x_{i}\}_{i=1}^N$ and a loss function $\ell(\theta, x_{i})$, we consider a training loss  $L(\theta):=\frac{1}{N} \sum_{i=1}^{N} \ell(\theta, x_{i})$ and a mini-batch loss
$L^{B}(\theta):= \frac{1}{B}\sum_{x_i \in \mathcal{B}}\ell(\theta, x_{i})$, where $\mathcal{B}$ is a randomly sampled subset of the training data such that $|\mathcal{B}| = B$.
With a minimum $\theta^* \in \R^d$, we define $\Delta L := \min_{\theta \in \partial D} L(\theta) - L(\theta^*)$ as a depth of a loss function around $\theta^*$.

We consider two types of stochastic gradient descent (SGD) methods; a discrete SGD and a continuous SGD.
Although the ordinary SGD is discrete, we focus on its continuous variation for mathematical convenience.

\textbf{Discrete SGD}:
Given an initial parameter $\theta_0 \in \mathbb{R}^d$ and a learning rate $\eta>0$, SGD generates a sequence of parameters $\{\theta_k\}_{k \in \mathbb{N}}$ by the following update rule:
\begin{align}
\label{eq:original_sgd}
\theta_{k+1}=\theta_{k}-\eta \nabla L^{B}(\theta_{k}), ~ k \in \mathbb{N}.
\end{align}
We model SGD as a gradient descent with a Gaussian noise perturbation.
We decompose $-\nabla L^{B}(\theta_{k})$ in (\ref{eq:original_sgd}) into a gradient term $-\nabla L(\theta_{k})$ and a noise term $
\nabla L(\theta_{k})- \nabla L^{B}(\theta_{k})$,
and model the noise as a Gaussian noise.
With this setting, the update rule in (\ref{eq:original_sgd}) is rewritten as
\begin{align}
\label{eq:gaussian_sgd}
 \theta_{k+1}=\theta_{k}
 - \eta\nabla L(\theta_{k})
 + \sqrt{\frac{\eta}{B}}W_{k}, 
\end{align}
where $W_{k} \sim N\left(0, \eta C(\theta_{k})\right)$ is a parameter-dependent Gaussian noise with its covariance.
Note that $\eta$ appears in both the covariance $\eta C(\theta_k)$ and the coefficient $\sqrt{\eta / B}$ of the noise $W_k$, because it is useful to make the connection with the subsequent continuous SGD clear. 

The Gaussianity of the noise on gradients is justified by the following reasons:
(i) if the batch size $B$ is sufficiently large, the central limit theorem ensures the noise term becomes Gaussian,
and (ii) several empirical studies show that the noise term follows Gaussian distribution \citep{Mandt2016-qz, Jastrzebski2017-gm,he2019control}.

\textbf{Continuous SGD}:
We formulate a continuous SGD such that its discretization corresponds to the discrete SGD (\ref{eq:gaussian_sgd}) based on the Euler scheme (e.g. Definition 5.1.1 of \cite{Gobet2016-gz}).
With a time index $t \geq 0$ and the given initial parameter $\theta_0 \in \mathbb{R}^d$, the continuous dynamic of SGD is written as the following system:
\begin{align}
 \dot{\theta}_t &= -\nabla L(\theta_t) +\sqrt{\frac{\eta}{ B}} {C(\theta_t)}^{1/2} \dot{w}_{t} \label{def:continuous_sgd}
\end{align}
where $w_{t}$ is a $d$-dimensional Wiener process, i.e. an $\mathbb{R}^d$-valued stochastic process with $t$ such that $w_0 = 0$ and $w_{t+u}-w_{t} \sim N(0, uI)$ for any $t,u > 0$.
We note that this system can be seen as a Gaussian perturbed dynamical system with a noise magnitude $\sqrt{{\eta} / { B}}$ because $\eta$ and $B$ do not evolve by time.
\subsection{Mean Exit Time}
We consider the problem on how discrete and continuous SGD's escape from minima of loss surfaces.
This is formally quantified by a notion of \textit{mean exit time}.
We define $\theta^* \in \mathbb{R}^d$ as a local minimum of loss surfaces,
and also define its neighborhood $D\subset \mathbb{R}^d$ as an open set which contains $\theta^*$.
We define the mean exit time as follows:
\begin{definition}[Mean exit time from $D$]\label{def:mean_exit_time}
Consider a continuous SGD (\ref{def:continuous_sgd}) starting from $\theta_0 \in D$.
Then, a mean exit time of the continuous SGD from $D$ is defined as
\begin{align*}
    \mathbb{E}[\tau] := \mathbb{E}[\min \left\{t: \theta_{t} \notin D\right\}].
\end{align*}
\end{definition}
Intuitively, a continuous SGD with small $\mathbb{E}[\tau]$ easily escapes from the neighbourhood $D$.

Similarly, we define the discrete mean exit time as follows.
Here, a product of the learning rate $\eta$ and the update index $k$ plays a role of the time index $t$, since the $\eta$ is regarded as a width of the discretization.
\begin{definition}[Discrete mean exit time from $D$]\label{def:discrete_mean_exit_step}
Consider a discrete SGD (\ref{eq:gaussian_sgd}) starting from $\theta_0 \in D$.
Then, a discrete mean exit time of the discrete SGD  from $D$ is defined as
\begin{align*}
    \mathbb{E}[\nu] := \mathbb{E}[\min \left\{k\eta: \theta_{k} \notin D\right\}].
\end{align*}
\end{definition}

In our study, we measure the escape efficiency by an inverse of the notion of mean exit time.
Rigorously, we consider the following definition:
$$
   \text{Escape efficiency} := \text{(Mean exit time)}^{-1}.
$$

\begin{remark}[Other measures on escape]
There exist several terms in the machine learning community
that represent similar notions.
The term ``escaping efficiency'' was first defined by \cite{Zhu2019-og} as $\mathbb{E}_{\theta_{t}}\left[L\left(\theta_{t}\right)-L\left(\theta_{0}\right)\right]$.
\cite{Xie2020-ty} defined an ``escape rate'' as a ratio between the probability of coming out from $\theta^*$'s neighborhood
and the probability mass around $\theta^*$.
They also defined an ``escape time'' by the inverse of the escape rate.
\end{remark}

\subsection{Setting and Basic Assumptions for SGD's Escape Problem}

We provide basic assumptions that are commonly used in the literature of the escape problem \citep{Mandt2016-qz,Zhu2019-og,Jastrzebski2017-gm,Xie2020-ty}.
\begin{assumption}[$L(\theta)$ is locally quadratic in $D$]
\label{assumption:quadratic}
There exists a matrix $H^* \in \mathbb{R}^{d \times d}$ such that for any $\theta \in D$, the following equality holds:
\begin{align*}
    \forall \theta\in D, L(\theta)=L\left(\theta^*\right)+\nabla L\left(\theta^*\right)\left(\theta-\theta^*\right)+\frac{1}{2}\left(\theta-\theta^*\right)^{\top} H^*\left(\theta-\theta^*\right)
\end{align*}
\end{assumption}
\begin{assumption}[Hesse covariance matrix]
    \label{assumption:strong_covariance}
    For any $\theta\in D$, $C(\theta)$ is approximately equal to $H^*$.
\end{assumption}

It is known that Assumption \ref{assumption:strong_covariance} holds when $\theta^*$ is a critical point \citep{Zhu2019-og,Jastrzebski2017-gm}.
It is also empirically shown that Assumption \ref{assumption:strong_covariance} can approximately hold even  for randomly chosen $\theta$ (see Section 2 of \cite{Xie2020-ty}).

Finally, we use the following definition as sharpness in our analysis.
\begin{definition}[Sharpness of a minimum $\theta^*$]
    \label{def:sharpness}
    Sharpness of $\theta^*$ is the maximum eigenvalue of $H^*$ in Assumption \ref{assumption:quadratic}, that is,
    \begin{align*}
        \lambda_{\max} = \lambda_{\max}(H^*).
    \end{align*}
\end{definition}
This is one of the most common definitions of sharpness \citep{jastrzebski2020break}, although the definition of sharpness is a controversial topic by itself \citep{Dinh2017-km},

%% file: quasi-potential_theory.tex
\section{Large Deviation Theory for SGD}
\label{sec:LDT_for_SGD}
\begin{figure}[t]
  \centering
  \vspace{0.3cm}
  \includegraphics[width=0.75\textwidth]{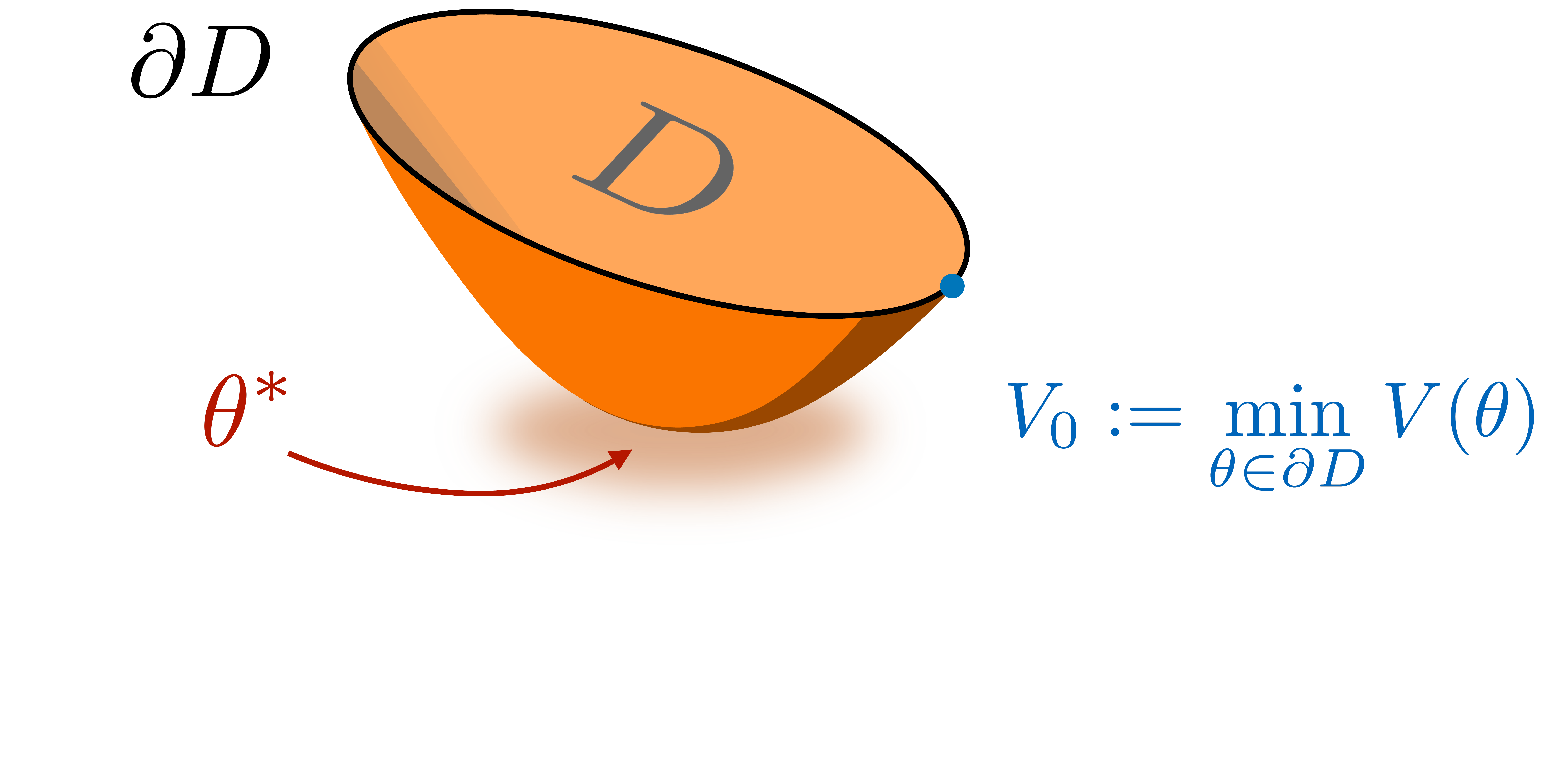}
  \vspace{-1.4cm}
  \caption{Notions related to the fundamental theorem (Theorem \ref{thm:basic_exit_time}), a minimum $\theta^*$,
  its neighborhood hood $D$ as well as the boundary $\partial D$, and $V_0$,
  which is defined as the smallest quasi-potential in $\partial D$.}
  \label{fig:definitions}
\end{figure}
We introduce the basic notions from the Large Deviation Theory \citep{Freidlin2012-iz, Dembo2010-fy}.

First, we define \textit{steepness} of a trajectory on a loss surface $L(\theta)$, followed by the continuous SGD (\ref{def:continuous_sgd}). 
Let $\varphi = \{\varphi_t\}_{t \in [0,T]} \subset \R^d$ be a trajectory in the parameter space over a time interval $[0,T]$ with a terminal time $T$, where $\varphi_t \in \R^d$ is a parameter which continuously changes in $t$ (see Figure \ref{fig:steepness}).
Also, $\varphi$ is regarded as a continuous map from $[0,T]$ to $\R^d$, i.e. is an element of $\mathbf{C}_{T}(\R^d)$
(a set of continuous trajectories in $\R^d$)
which is a support of continuous SGD during $[0,T]$.
Given a trajectory $\varphi$ and the system (\ref{def:continuous_sgd}), we define the following quantity:
\begin{definition}[Steepness of $\varphi$]\label{def:steepness} Steepness of a trajectory $\varphi$ followed by (\ref{def:continuous_sgd}) is defined as
  $$
    S_{T}(\varphi) := \frac{1}{2}\int_{0}^{T}\left(\dot{\varphi}_t+\nabla L\left(\varphi_t\right)\right)^{\top} C\left(\varphi_t\right )^{-1/2}\left(\dot{\varphi}_t+\nabla L\left(\varphi_t\right)\right) d t.
  $$
\end{definition}
Intuitively, steepness $S_{T}(\varphi)$ is interpreted as the hardness for the system (\ref{def:continuous_sgd}) to follow this trajectory $\varphi$ up the hill on $L(\theta)$,
as illustrated in Figure \ref{fig:steepness}.
This notion is generally utilized in the Large Deviation Theory,
and is called ``normalized action functional'' in \citet[Section 3.2]{Freidlin2012-iz}
or ``rate function'' in \citet[Section 1.2]{Dembo2010-fy}.

Steepness is a useful measure
to formally describe a distribution of trajectories generated by continuous SGD.
If a trajectory $\varphi$ has a large steepness $S_{T}(\varphi)$,
the probability that the system takes the trajectory decreases exponentially.
Formally, the distributions are described as follows.
$\mathrm{P}_{\varphi''}\left( \cdot \right)$ is a distribution of a (random) trajectory $\varphi''$ generated from the continuous SGD (\ref{def:continuous_sgd}).

\begin{lemma}[Theorem 3.1 in  \cite{Freidlin2012-iz}]\label{lemma:action_functional_1}
  For any $\delta,\zeta >0$ and $ \varphi \in \mathbf{C}_{T}\left(\R^{d}\right)$, there exists  $\varepsilon = \varepsilon(\delta,\zeta)
    > 0$ such that the following holds:
  $$
    \mathrm{P}_{\varphi''}\left(\varphi''\in \left\{ \varphi' \in \mathbf{C}_{T}(\R^d)\mid \rho\left(\varphi^\prime, \varphi\right)<\delta \right\}\right) \geq \exp \left\{-\varepsilon^{-2}\left[S_{T}\left(\varphi\right)+\zeta\right]\right\},
  $$
  where $\rho(\varphi^\prime, \varphi)=\sup_{t\in [0,T]} \|\varphi^\prime_t-\varphi_t\| $.
\end{lemma}

\begin{lemma}[Theorem 3.1 in  \cite{Freidlin2012-iz}]\label{lemma:action_functional_2}
  Let $\Phi(s)=\left\{\varphi \in \mathbf{C}_{T}(\R^d)\mid S_{T}(\varphi) \leq s\right\}$. For any $\delta,\zeta$ and $s >0$, there exists $\varepsilon = \varepsilon(\delta, \zeta, s) > 0$ such that the following holds:
  $$
    \mathrm{P}_{\varphi''}\Big(\varphi'' \in \left\{ \varphi' \in  \mathbf{C}_{T}(\R^d) \mid \rho(\varphi^\prime, \Phi(s)) \geq \delta \right\}\Big) \leq \exp \{-\varepsilon^{-2}(s-\zeta)\},
  $$
  where $\rho(\varphi^\prime, \Phi(s))=\inf_{\varphi \in \Phi(s)} \rho(\varphi^\prime,\varphi)$.
\end{lemma}
Although we focus on the continuous SGD,
similar discussions are applicable to a general class of diffusion processes \citep[Section 5.7]{Dembo2010-fy} and systems with Markov perturbations \citep[Section 6.5]{Freidlin2012-iz}.

We secondly define \textit{quasi-potential}, which is the smallest steepness from a minimum $\theta^*$ to a boundary $\partial D$.
It plays an essential role in the mean exit time.
\begin{definition}[Quasi-potential]\label{def:quasi_potential}
  Given the system (\ref{def:continuous_sgd}) whose initial point is a local minima $\theta^*$, \textit{quasi-potential} of a parameter $\theta \in D$ is defined as
  $$
    V(\theta):=\inf_{T>0}
               \inf_{\varphi: \substack{(\varphi_{0}, \varphi_T)=(\theta^*, \theta)}} S_{T}(\varphi).
  $$
\end{definition}
Same as steepness, quasi-potential can be seen as the minimum effort the system (\ref{def:continuous_sgd}) needs to climb from $\theta^*$ up to $\theta$ on $L(\theta)$.
(For more details, see \citet[Section 5.3]{Freidlin2012-iz}).

We describe the mean exit time of a continuous SGD (\ref{def:continuous_sgd}) based on the notion of quasi-potential.
We obtain the following theorem by applying the fundamental results of Large Deviation Theory:
\begin{theorem}[Fundamental Theorem] \label{thm:basic_exit_time}
    Consider the continuous SGD (\ref{def:continuous_sgd}) whose initial point is the local minima $\theta_0 = \theta^*$.
    Suppose Assumption \ref{assumption:quadratic} holds.
    Then, the mean exit time (Definition \ref{def:mean_exit_time}) has the following limit:
  \begin{align*}
    \lim_{\eta\to 0}\frac{\eta}{B}\ln\mathbb{E}\left[ \tau\right]
    = V_0
  \end{align*}
  holds,
  where $V_0 := \min_{\theta \in \partial D} V(\theta)$.
\end{theorem}
We obtain Theorem \ref{thm:basic_exit_time} by adapting a general theorem in \cite[Section 4]{Freidlin2012-iz} to our setting with Assumption \ref{assumption:quadratic}.
Rigorously, we verify that several requirements of the general theorem, such as asymptotic stability and attractiveness, are satisfied with our setup.
The precise description of the assumptions can be found in Appendix \ref{appendix:stability_attraction}, and the proof of Theorem \ref{thm:basic_exit_time} under our setup
in Appendix \ref{appendix:basic_exit_time}.

%% file: main_results.tex
\section{Mean Exit Time Analysis}
\label{sec:main_analysis}

In this section, we give an asymptotic analysis of the mean exit time as our main result.
As preparation, we provide an approximate computation of the quasi-potential in our setting, then we give the main theorem.

\subsection{Approximate Computation of Quasi-potential}
\label{sec:computation_of_qp}

We develop an approximation of the  quasi-potential $V(\theta)$, which is necessary to study the mean exit time by the fundamental theorem (Theorem \ref{thm:basic_exit_time}).
However, the direct calculation with a general $C(\theta)$ is a difficult problem,
and at best we get a necessary condition for the exact formula (Appendix \ref{appendix:quasi_potential}).
Instead, we consider a \textit{proximal system} which is a simplified version of the continuous SGD (\ref{def:continuous_sgd}) with a state-independent noise covariance.

\subsubsection{Proximal System with $C(\theta) = I$}
We define the following proximal system which generates a sequence $\{\hat{\theta}_t\}$:
\begin{align}
  \label{eq:sgld}
  &\dot{\hat{\theta}}_{t}=-\nabla L\left(\hat{\theta}_{t}\right)
                                  + \sqrt{\frac{\eta}{B}}
                                    \dot{w}_{t}
\end{align}
This system is obtained by replacing the covariance $C(\theta)$ of the continuous SGD (\ref{def:continuous_sgd}) into an identity $I$.
That is, this proximal system is regarded as a Gaussian gradient descent with isotropic noise.

We further define steepness and quasi-potential of the proximal system as follow:
\begin{align}
  &\textbf{Steepness}\quad\text{For each $\varphi\in \mathbf{C}_{T}(\R^d)$, }\hat{S}_{T}(\varphi):=\frac{1}{2}
                                           \int_{0}^{T}
                                             \|\dot{\varphi}_{t}+\nabla L\left(\varphi_{t}\right)\|^2
                                           dt \\
  &\textbf{Quasi-potential}\quad\text{For each $\theta\in D$, }\hat{V}(\theta):=\inf _{T>0}
                                            \inf _{\varphi:\left(\varphi_{0}, \phi_{r}\right)=\left(\theta^{*}, \theta\right)}
                                            \hat{S}_{T}(\varphi)
\end{align}
Owing to the noise structure of the proximal system, we achieve an simple form of the quasi-potential.
For the quasi-potential $\hat{V}(\theta)$, the following lemma holds:
\begin{lemma} \label{lemma:calculate_vhat}
Under Assumption \ref{assumption:quadratic},
$\hat{V}(\theta) = 2\left(L(\theta) - L(\theta^*)\right)$.
\end{lemma}
\begin{proof}
If the function $\varphi_{t}$ for $t \in\left[0, T\right]$ does not exit from $D \cup \partial D$,
\begin{align*}
  \hat{S}_{T}(\varphi)
    &= \frac{1}{2} \int_{0}^{T}\left\|\dot{\varphi}_{t}-\nabla L\left(\varphi_{t}\right)\right\|^{2} dt
        + 2 \int_{0}^{T}\dot{\varphi}_{t}^\top\nabla L\left(\varphi_{t}\right) dt\\
    &= \frac{1}{2} \int_{0}^{T}\left\|\dot{\varphi}_{t}-\nabla L\left(\varphi_{t}\right)\right\|^{2} dt
        + 2 \left(L(\varphi_{T}) - L(\varphi_{0})\right)\\
    & \geq 2 \left(L(\varphi_{t}) - L(\varphi_{0})\right)
\end{align*}
The equality holds when $\dot{\varphi}_{t}=\nabla L\left(\varphi_{t}\right)$.
Since quasi-potential at $\theta$ is the infimum of the steepness from $\theta^*$ to $\theta$,
$\hat{V}(\theta) = 2 \left(L(\theta) - L(\theta^*)\right)$ is obtained.
\end{proof}
Lemma \ref{lemma:calculate_vhat} shows that the quasi-potential with the proximal system
is simply represented as the height of $\theta$ from a minimum $\theta^*$. 
By the quasi-potential, we simply obtain the following result by combining Theorem \ref{thm:basic_exit_time}:
\begin{proposition}[Mean Exit Time of Proximal System] \label{thm:proxy_reuslt}
Consider the proxy system (\ref{eq:sgld})
whose initial point is the local minima $\theta_0 = \theta^*$.
Suppose that Assumption
\ref{assumption:quadratic},
\ref{assumption:strong_covariance},
and \ref{assumption:trajectory_grad_bound} hold.
    Then, the mean exit time of (\ref{eq:sgld}) from the neighborhood $D$, $\mathbb{E}[\hat{\tau}]$,
    has the following limit
    \begin{align*}
      \lim_{\eta \to 0}\frac{\eta}{B}\ln\mathbb{E}[\hat{\tau}] = \Delta L.
    \end{align*}
\end{proposition}

\subsubsection{Approximation of Quasi-potential $V_0 := \min_{\theta \in \partial D} V(\theta)$}
We approximate the target quasi-potential $V(\theta)$ using $\lambda_\mathrm{max}^{-{1} / {2}}\hat{V}(\theta)$ from the proximal system.
For this sake, we impose the following assumption:
\begin{assumption}
  \label{assumption:trajectory_grad_bound}
  There exists $K>0$ such that for any $\varphi\in \mathbf{C}_{T}\left(\mathbb{R}^{d}\right)$ and $t \in [0,T]$, $\dot{\varphi}_t \leq K$ holds.
\end{assumption}
This claims that the velocities of trajectories do not become infinitely large.
With this mild assumption, we obtain the estimation of $V_0$ as follows.
We define that $\kappa$ is the condition number of $C(\theta^*) := \lambda_{\max}/\lambda_{\min}$.
\begin{lemma}
  \label{lemma:quasi-potential_estimate}
  Under Assumption \ref{assumption:quadratic}, \ref{assumption:strong_covariance}, and \ref{assumption:trajectory_grad_bound},
  there exists a constant $A$ such that
  \begin{align*}
    \left|V_0 -\lambda_\mathrm{max}^{-\frac{1}{2}}\hat{V}_0\right|
    \leq
    A \lambda_\mathrm{min}^{\frac{1}{2}}\left(\kappa^\frac{1}{2} - 1\right).
  \end{align*}
\end{lemma}
This lemma implies that the quasi-potential of continuous SGD
is approximated by $\lambda_\mathrm{max}^{-{1} / {2}}\hat{V}_0$.
When $\lambda_\mathrm{max} > 1$ holds, continuous SGD has smaller quasi-potential than that of the proximal system.
We can see that the tightness of the approximation is described by by the degree of ``anisotropy'' of the noise (i.e. $\kappa$), since the bound $A \lambda_\mathrm{min}^{{1} / {2}}(\kappa^{{1} / {2}} - 1)$ is mainly determined $\lambda_{\mathrm{max}}$. 
 
\subsection{Main Results: Mean Exit Time Analysis}
\label{sec:main_results}
As our main results, we give inequalities that characterize a limit of the mean escape time.
We recall the definition of the depth of a minimum $\theta^*$ as $\Delta L := \min_{\theta \in \partial D} L(\theta) - L(\theta^*)$.

\paragraph{Continuous SGD}
First, we study the case of continuous SGD (\ref{def:continuous_sgd}).
This result is obtained immediately by combining the fundamental theorem (Theorem \ref{thm:basic_exit_time}) with the approximated quasi-potential (Lemma \ref{lemma:quasi-potential_estimate}):
\begin{theorem}[Mean Exit Time of Continuous SGD] \label{thm:result}
Consider the continuous SGD (\ref{def:continuous_sgd}) whose initial point is the local minima $\theta_0 = \theta^*$.
Suppose that Assumption
\ref{assumption:quadratic},
\ref{assumption:strong_covariance},
and \ref{assumption:trajectory_grad_bound} hold.
    Then, the mean exit time (Definition \ref{def:mean_exit_time}) from the neighbourhood $D$ has the following limit:
      \begin{align*}
            2\lambda_\mathrm{max}^{-\frac{1}{2}}\Delta L
            -
            A \lambda_\mathrm{min}^{\frac{1}{2}}\left(\kappa^\frac{1}{2} - 1\right)
      \leq &\lim_{\eta \to 0}\frac{\eta}{B}\ln\mathbb{E}[\tau]
      \leq 2\lambda_\mathrm{max}^{-\frac{1}{2}}\Delta L
           +
           A \lambda_\mathrm{min}^{\frac{1}{2}}\left(\kappa^\frac{1}{2} - 1\right).
    \end{align*}
\end{theorem}
Excluding the effect of the approximation $A \lambda_\mathrm{min}^{{1} / {2}}(\kappa^{{1} / {2}} - 1)$,
this result indicates that continuous SGD needs $\exp (2\frac{B}{\eta}\Delta L\lambda_{\mathrm{max}}^{-{1} / {2}} )$ number of steps asymptotically, before escaping from the neighborhood $D$ of the local minima $\theta^*$.
Compared to the quasi-potential of the proxy system $2\Delta L$ (Proposition \ref{thm:proxy_reuslt}),   
the covariance matrix reduces quasi-potential in the factor of $\lambda_{\mathrm{max}}^{{1} / {2}}$.
This result endorses the fact that SGD's noise structure, $C(\theta)$, exponentially accelerates the escaping \citep{Xie2020-ty},
because quasi-potential exponentially affect mean exit time (Theorem \ref{thm:basic_exit_time}).
A more rigorous comparison is given in Section \ref{sec:comparison}.

\paragraph{Discrete SGD}
Next, we give the mean escape time analysis for discrete SGD (\ref{eq:gaussian_sgd}). 
Our approach is to combine the following discretization error analysis to the continuous SGD results (Theorem \ref{thm:result}):
\begin{lemma}[Discretization Error]
  \label{lemma:discretization}
  For a stochastic system with Gaussian perturbation and its discrete correspondence,
  the discretization error of exit time has the following convergence rate
  \begin{align*}
  \mathbb{E}[\nu]  - \mathbb{E}[\tau] = \mathcal{O}(\sqrt{\eta}).
  \end{align*}
\end{lemma}
The following lemma can be simply derived as a special case of \citep[Theorem 17]{Gobet2010-qx} by substituting $g(\cdot) = 0, f(\cdot) = 1,$ and  $k(\cdot) = 0$ in their definition.

Based on the analysis, we obtain the following result:
\begin{theorem}[Mean Exit Time of Discrete SGD]
  \label{thm:result2}
Consider the discrete Gaussian SGD (\ref{eq:gaussian_sgd}) whose initial point is the local minima $\theta_0 = \theta^*$.
Suppose that Assumption
\ref{assumption:quadratic},
\ref{assumption:strong_covariance},
and \ref{assumption:trajectory_grad_bound},
hold.
    Then, the mean exit time (Definition \ref{def:mean_exit_time}) from the neighbourhood $D$ has the following limit:
    \begin{align*}
            2\lambda_\mathrm{max}^{-\frac{1}{2}}\Delta L
            -
            A \lambda_\mathrm{min}^{\frac{1}{2}}\left(\kappa^\frac{1}{2} - 1\right)
      \leq &\lim_{\eta \to 0}\frac{\eta}{B}\ln\mathbb{E}[\nu]
      \leq 2\lambda_\mathrm{max}^{-\frac{1}{2}}\Delta L
           +
           A \lambda_\mathrm{min}^{\frac{1}{2}}\left(\kappa^\frac{1}{2} - 1\right).
    \end{align*}
\end{theorem}
This result indicates that continuous and discrete SGD have the identical asymptotic the mean exit times.
In other words, the discretization error is asymptotically negligible in this analysis of escape time.
Fig \ref{fig:sgd_models} summarizes the whole structure of our results.

\begin{figure}[t]
\begin{adjustwidth}{-1.2cm}{}
    \centering
    \vspace{-0.2cm}
    \begin{tikzpicture}[every node/.style={inner xsep = 2.5mm, inner ysep = 2.5mm}]
        \matrix [column sep=8mm, row sep=5.5mm] {
          &&\node (sgd) [draw, fill={rgb,255:red,255; green,255; blue,255}, shape=rectangle, align=center, rounded corners=2ex] {Original SGD \\
            $\theta_{k+1}=\theta_{k}-\eta \nabla L^{B}(\theta_{k})$
            (\ref{eq:original_sgd})
          }; \\
           & & \\
           & & \\
          \node (idsgd) [draw, fill={rgb,255:red,213; green,213; blue,213}, shape=rectangle, align=center, rounded corners=2ex] {Discrete Proxy System \\\\
          $\dot{\theta}_{k+1}=-\nabla L\left(\theta_{k}\right) +\sqrt{\frac{\eta}{B}} N\left(0, \eta I\right)$
          }; &
          \node (map1) [draw, shape=rectangle, align=center] {$C(\theta)\mapsto I$}; &
          \node (dsgd) [draw, fill={rgb,255:red,250; green,174; blue,108}, shape=rectangle, align=center, rounded corners=2ex] {Discrete SGD \\\\
          $\theta_{k+1}=\theta_{k} - \eta\nabla L(\theta_{k}) + \sqrt{\frac{\eta}{B}}W_{k}$ (\ref{eq:gaussian_sgd})\\\\ Exit time: Theorem \ref{thm:result2}}; \\
           & & \\
           & & \\
           & & \\
           & & \\
          \node (icsgd) [draw, fill={rgb,255:red,213; green,213; blue,213}, shape=rectangle, align=center, rounded corners=2ex] {Proxy System \\\\
          $\dot{\theta}_{t}=-\nabla L\left(\theta_{t}\right) +\sqrt{\frac{\eta}{B}} \dot{w}_{t}$ (\ref{eq:sgld})\\\\ Exit time: Proposition \ref{thm:proxy_reuslt}}; &
          \node (map2) [draw, shape=rectangle, align=center] {$C(\theta)\mapsto I$}; &
          \node (csgd) [draw, fill={rgb,255:red,250; green,174; blue,108}, shape=rectangle, align=center, rounded corners=2ex] {Continuous SGD \\\\
          $\dot{\theta}_t = -\nabla L(\theta_t) +\sqrt{\frac{\eta}{ B}} {C(\theta_t)}^{1/2} \dot{w}_{t}$ (\ref{def:continuous_sgd})\\\\ Exit time: Theorem \ref{thm:result}}; \\
        };
        \draw[<->, ultra thick] (idsgd) to node[left, align=left]{
            Convergence of exit time\\
            $\mathcal{O}(\sqrt{\eta})$ (Lemma \ref{lemma:discretization})} (icsgd);
        \draw[<->, ultra thick] (dsgd) to node[left, align=left]{
            Convergence of exit time\\
            $\mathcal{O}(\sqrt{\eta})$ (Lemma \ref{lemma:discretization})} (csgd);
        \draw[->, ultra thick] (map1) -- (idsgd);
        \draw[-, ultra thick] (map1) -- (dsgd);
        \draw[-, ultra thick] (map2) -- (csgd);
        \draw[->, ultra thick] (map2) -- (icsgd);
        \draw[->, dotted, ultra thick] (sgd) to node[left, align=left]{Convergence of\\ coefficient ($B\to\infty$)} (dsgd);
    \end{tikzpicture}
    \vspace{0.3cm}
    \caption{The whole structure of our results.}
    \label{fig:sgd_models}
\end{adjustwidth}
\end{figure}
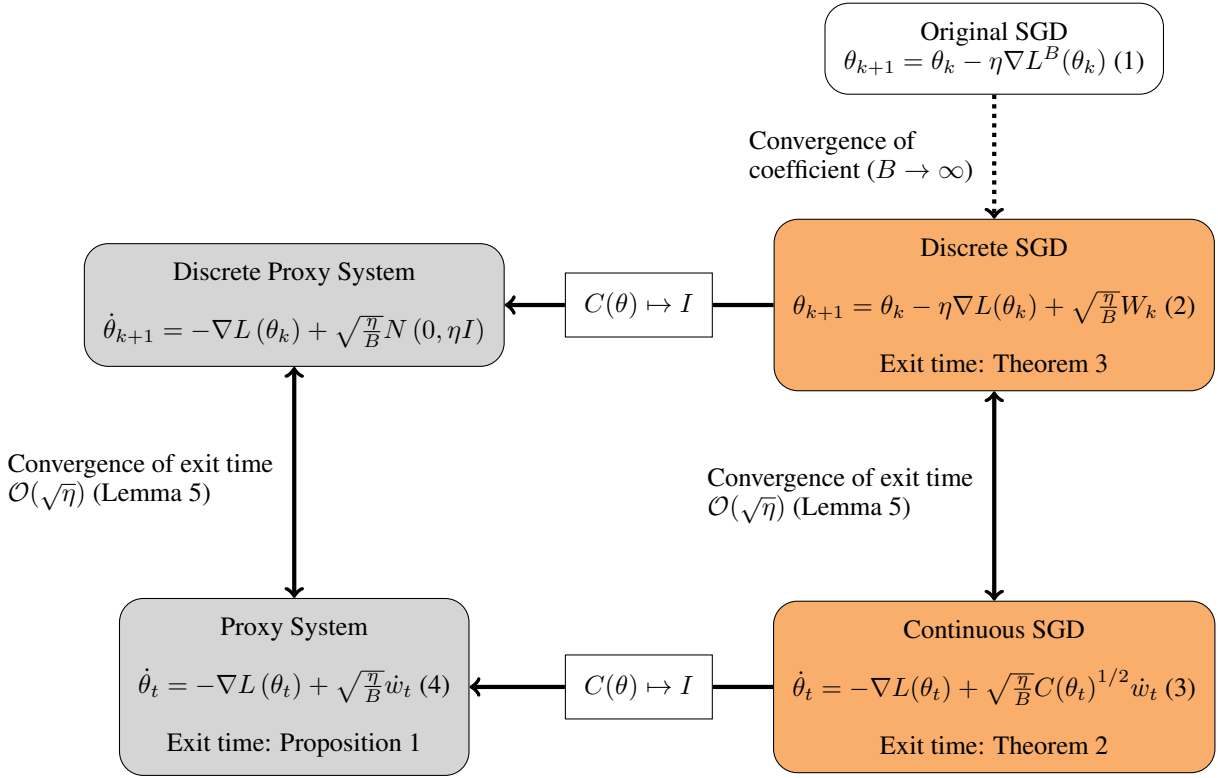

\begin{proof}
For the continuous SGD, by Theorem \ref{thm:basic_exit_time}, we have 
$\lim_{\eta\to 0}\frac{\eta}{B}\ln\mathbb{E}[\tau] = V_0$,
where $V_0 = \min _{\theta^{\prime} \in \partial D} V\left(\theta^{\prime}\right)$.
With this result,
it remains to evaluate the discretization error of exit time.

Here, without loss of generality, we assume $\mathbb{E}[\nu ] > 1$.
Also, we consider a case with $\mathbb{E}[\nu] - \mathbb{E}[\tau] \geq 0$.
For the opposite case $\mathbb{E}[\nu] - \mathbb{E}[\tau] < 0$, we can obtain the same result by repeating the following proof.
By Lemma \ref{lemma:discretization}, for sufficiently small $\eta$, there exists a constant $c$ such that 
$0 \leq \mathbb{E}[\nu]  - \mathbb{E}[\tau] \leq c\sqrt{\eta}$ holds.
Therefore, the discrete exit time can be lower bounded as
\begin{align*}
\frac{\eta}{B}\ln\mathbb{E}[\nu]
\geq \frac{\eta}{B}\ln\mathbb{E}[\tau],
\end{align*}
and also upper bounded as
\begin{align*}
  \frac{\eta}{B}\ln\mathbb{E}[\nu]
  &\leq  \frac{\eta}{B}\ln\left(\mathbb{E}[\tau] + c\sqrt{\eta}\right) \\
  &=  \frac{\eta}{B}\ln\left(1 + \mathbb{E}[\tau] - 1 + c\sqrt{\eta}\right) \\
  &\leq   \frac{\eta}{B}\ln\left( \mathbb{E}[\tau]\right)
        + \frac{\eta}{B}\ln\left( 1 + c\sqrt{\eta}\right).
\end{align*}
The last inequality follows that $\log (1+a + b) \leq \log (1+a) + \log(1 + b)$ for any $a,b > 0$.
Using the lower and upper bound, we obtain
\begin{align*}
\lim_{\eta \to 0}\frac{\eta}{B}\ln\mathbb{E}[\nu] = 
\lim_{\eta \to 0}\left\{\frac{\eta}{B}\ln\left(\mathbb{E}[\tau]\right)
        + \frac{\eta}{B}\ln\left(1 + c\sqrt{\eta}\right)\right\} = V_0.
\end{align*}
Combined with Lemma \ref{lemma:quasi-potential_estimate} and  Theorem  \ref{thm:basic_exit_time}, we obtain the statement of Theorem \ref{thm:result2}.
\end{proof}

%% file: experiment.tex
\section{Numerical Validation}
\label{sec:experiment}
We provide numerical experiments to validate our result
under practical scenarios. 
We use a multi-layer perceptron with one hidden layer with 5000 units,
mean square loss function,
fed with the AVILA dataset \citep{de2011method}.
To obtain the local minimum $\theta^*$,
we run the gradient descent network for a sufficiently long time (1000 epochs)
to obtain asymptotically stable $\theta^*$.
The region $D$ is defined as a neighborhood of $\theta^*$.
With $\theta^*$ as an initial value, we measure 
the exit times with SGD for 100 times independently.
We measure the average number of steps
at which SGD exits from $D$ as the discrete mean exit time.
To observe the dependency on the essential hyper-parameters ($\lambda_{\mathrm{max}}$, $\eta$, $B$, and $\Delta L$,), 
we compute the Pearson correlation coefficient, i.e. the linear correlation.
The sharpness of $\theta^*$ is controlled
by mapping $L(\theta)$ to $L(\sqrt{\alpha} \theta)$ with a parameter $\alpha>0$.
Since this mapping changes $\lambda_{\mathrm{max}}$ to $\alpha\lambda_{\mathrm{max}}$ with other properties remaining the same, we use $\alpha$ as a surrogate of the sharpness $\lambda_{\mathrm{max}}$.
In the similar manner, $\Delta L$ is controlled by mapping $L(\theta)$ to $\beta L(\theta)$,  where, $\beta$ is a surrogate of the depth of a minimum $\Delta L$.

Fig. \ref{fig:experiment} shows
the discrete mean exit time has exponential dependency
on $\lambda_{\mathrm{max}}^{-1/2}$, $\eta^{-1}$, $B$, and $\Delta L$,
which is aligned with Theorem \ref{thm:result2}.
As a reference,
we provide the same experiment with $C(\theta) \mapsto I$ (i.e. (\ref{eq:sgld}))
Fig. \ref{fig:reference} shows the discrete mean exit time is independent of sharpness
while $\eta$ and $\Delta L$ show the same trend (Proposition \ref{thm:proxy_reuslt}).
All the codes are available. 
\footnote{
 Source code for experiments \url{https://github.com/ibayashi-hikaru/MSML_experiments}.
}

\begin{figure}[H]
\centering
\begin{minipage}[c]{0.33\textwidth}
    \includegraphics[width=\textwidth]{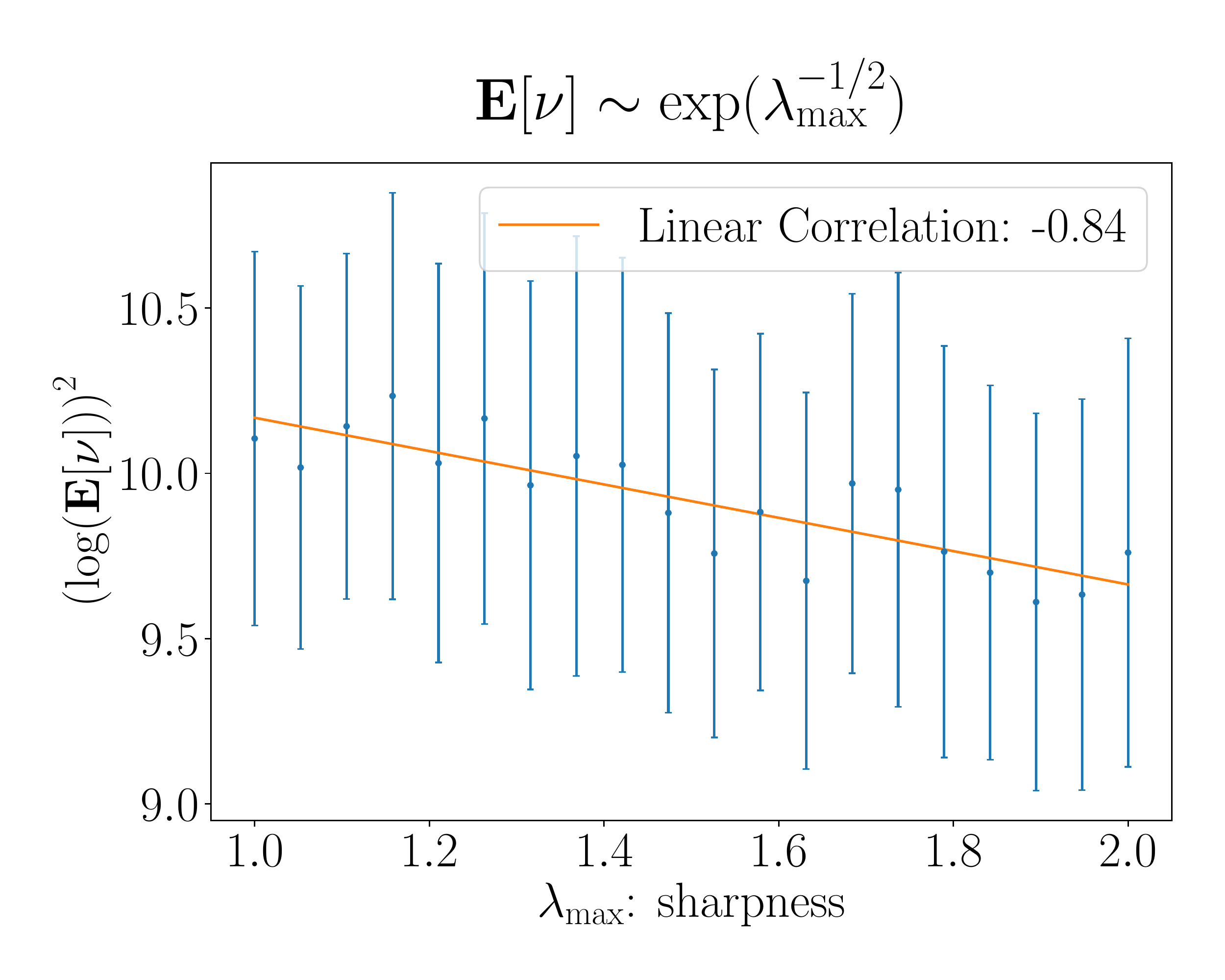}
\end{minipage}
\begin{minipage}[c]{0.33\textwidth}
    \includegraphics[width=\textwidth]{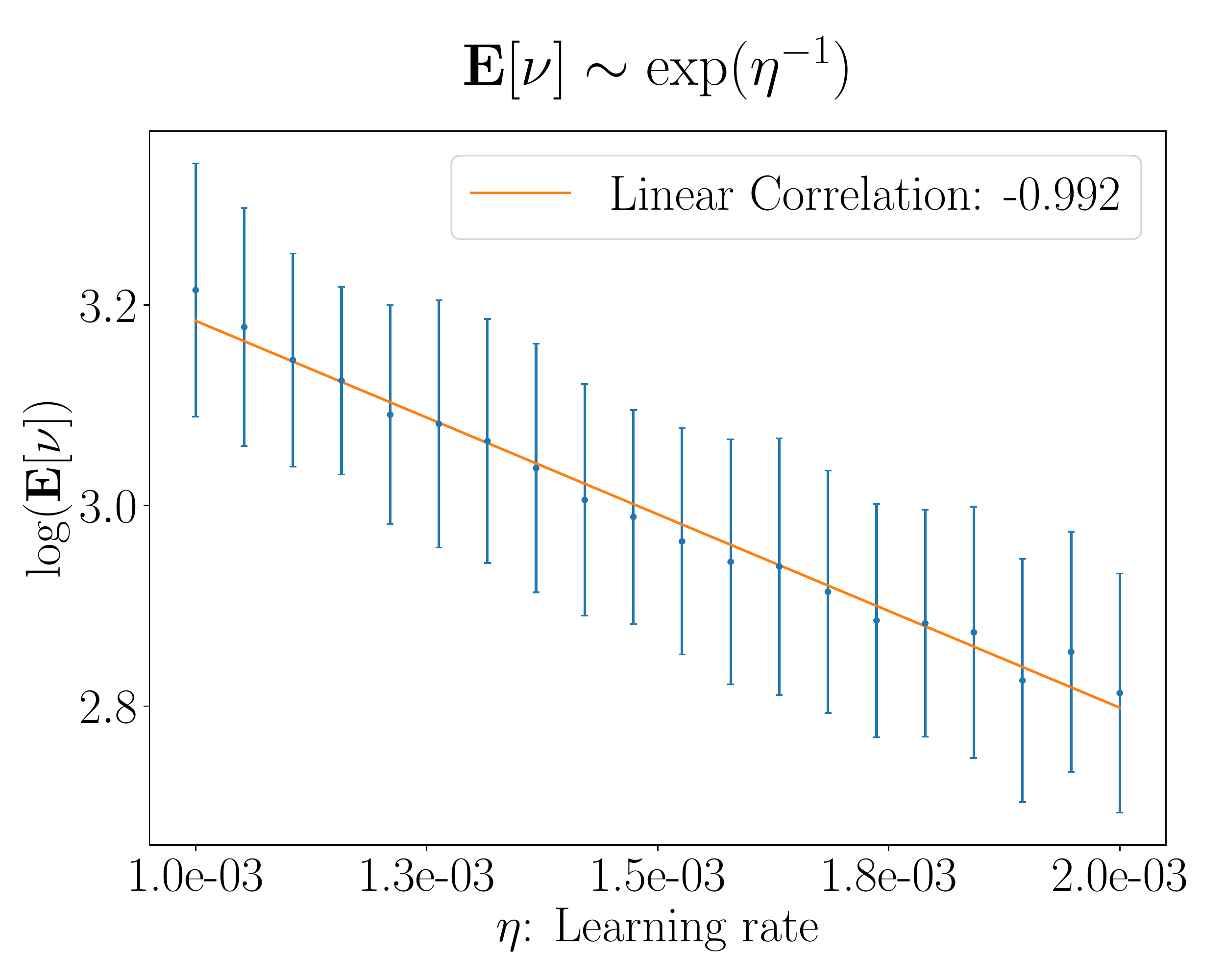}
\end{minipage}\\
\begin{minipage}[c]{0.33\textwidth}
    \includegraphics[width=\textwidth]{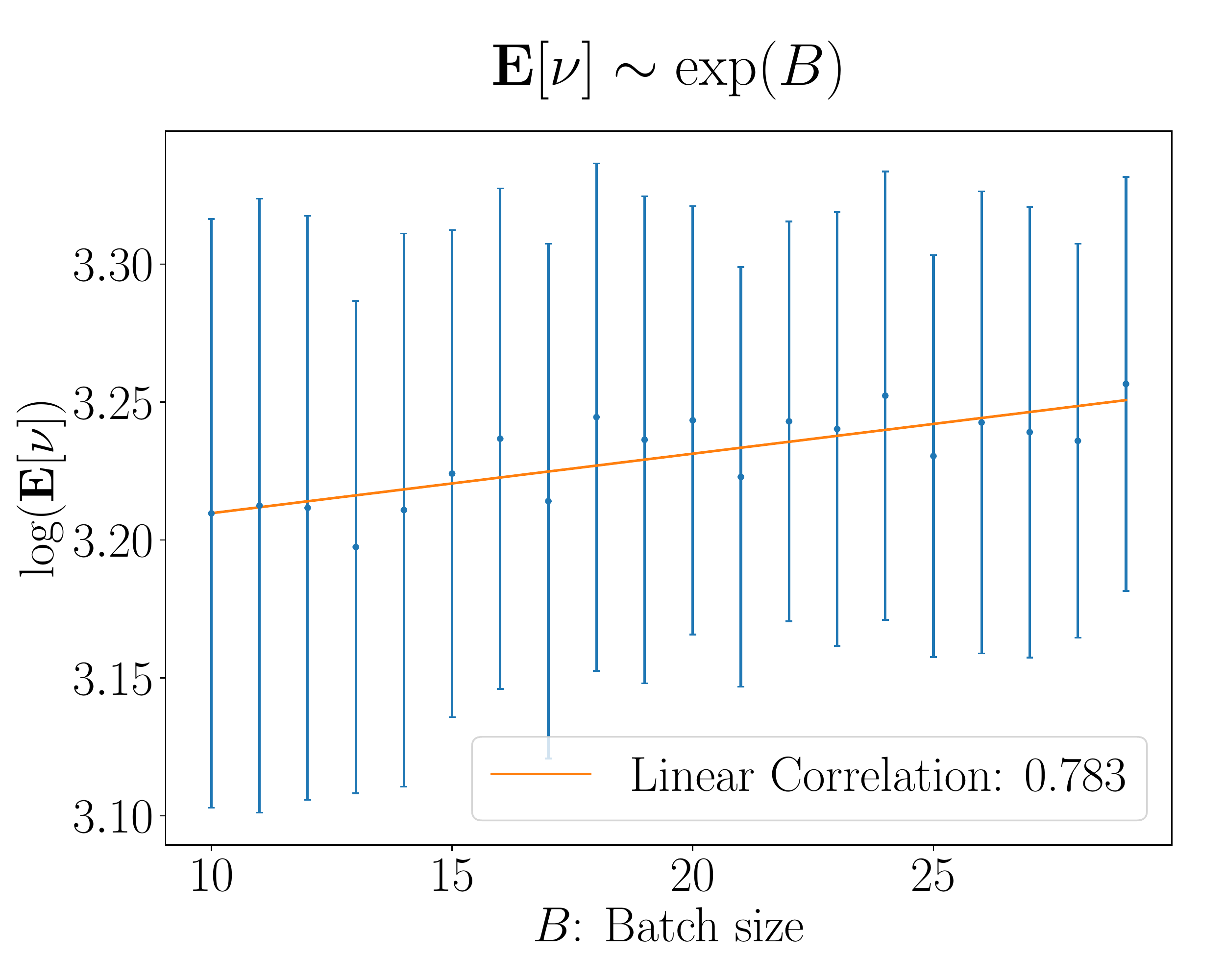}
\end{minipage}
\begin{minipage}[c]{0.33\textwidth}
    \includegraphics[width=\textwidth]{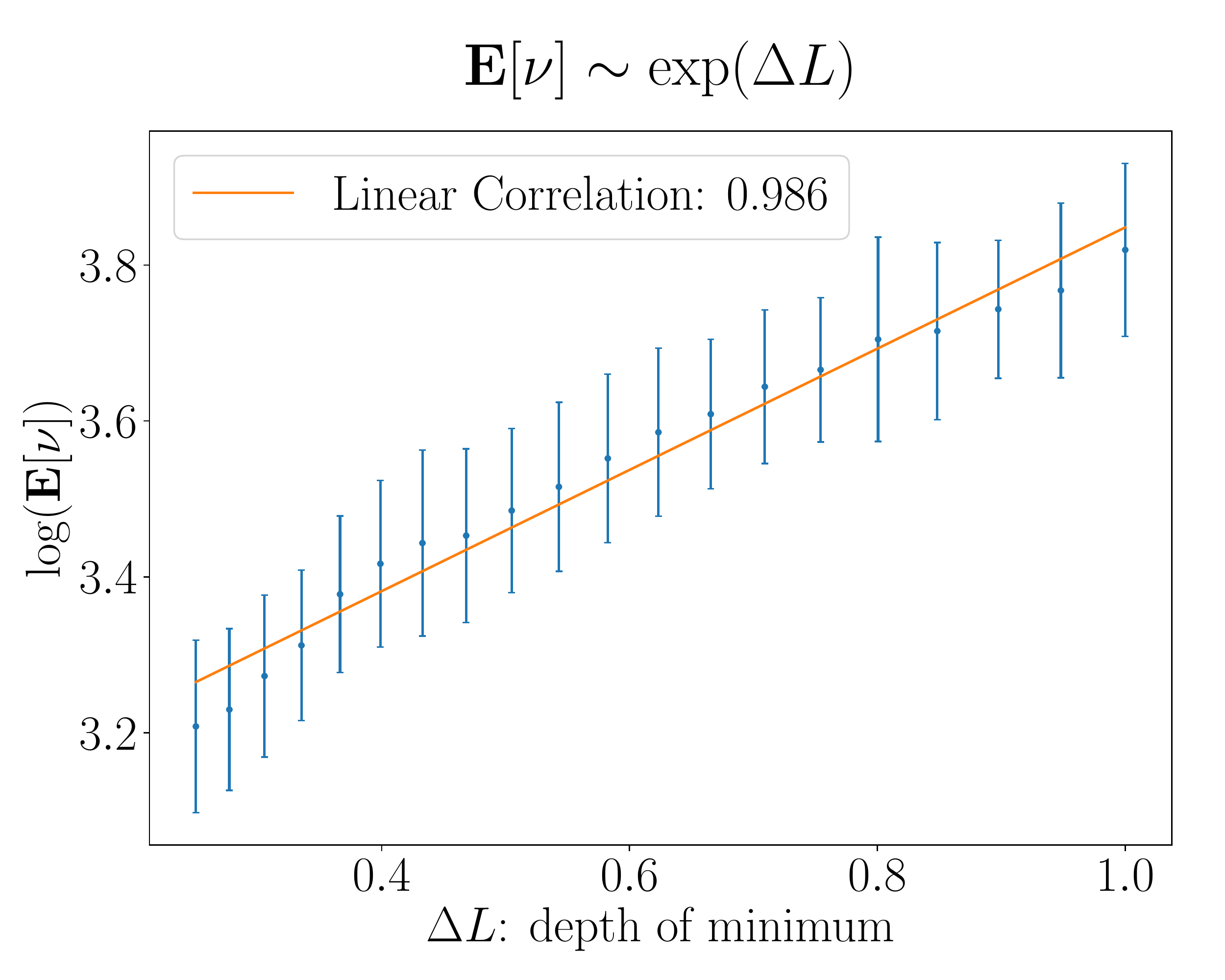}
\end{minipage}
\caption{Numerical validation of Theorem \ref{thm:result2},
where the mean exit time shows exponential dependency on $\lambda_\mathrm{max}^{{1} / {2}}$,
$\eta^{-1}$, $B$, and $\Delta L$. The error bars indicate the standard deviation.}
\label{fig:experiment}
\end{figure}
\begin{figure}[H]
\centering
\begin{minipage}[c]{0.325\textwidth}
\centering
    \includegraphics[width=\textwidth]{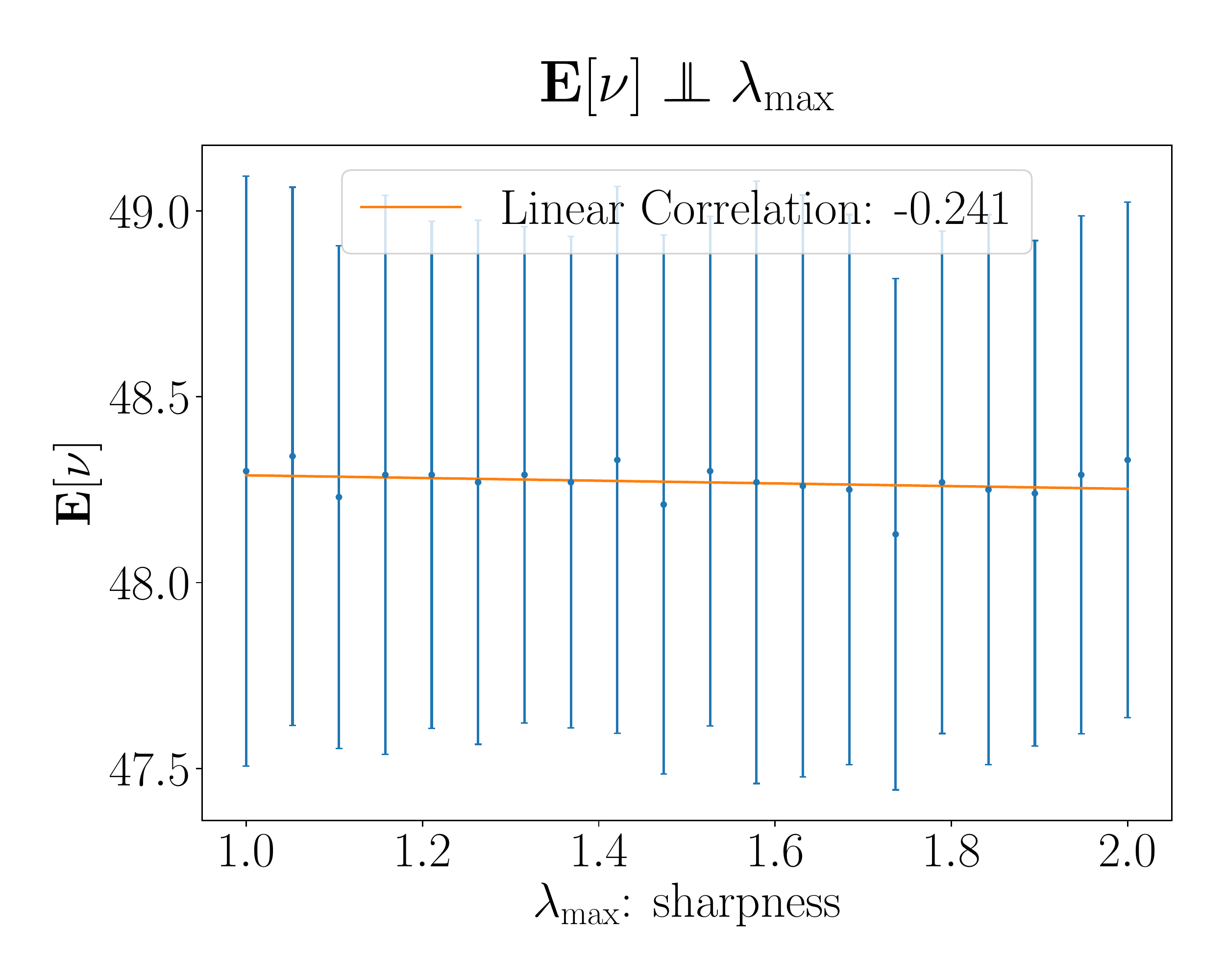}
\end{minipage}
\begin{minipage}[c]{0.325\textwidth}
\centering
    \includegraphics[width=\textwidth]{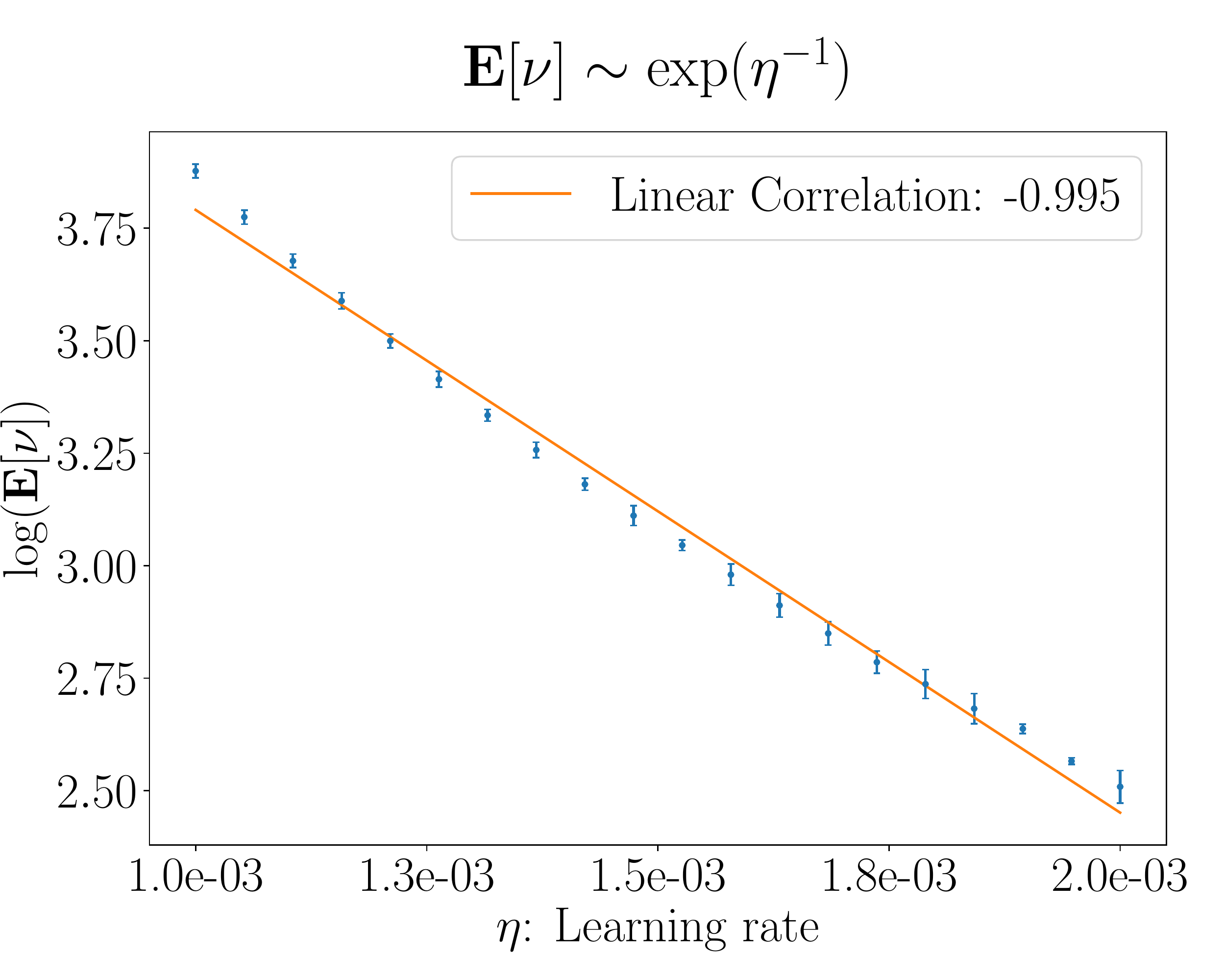}
\end{minipage}
\begin{minipage}[c]{0.325\textwidth}
\centering
    \includegraphics[width=\textwidth]{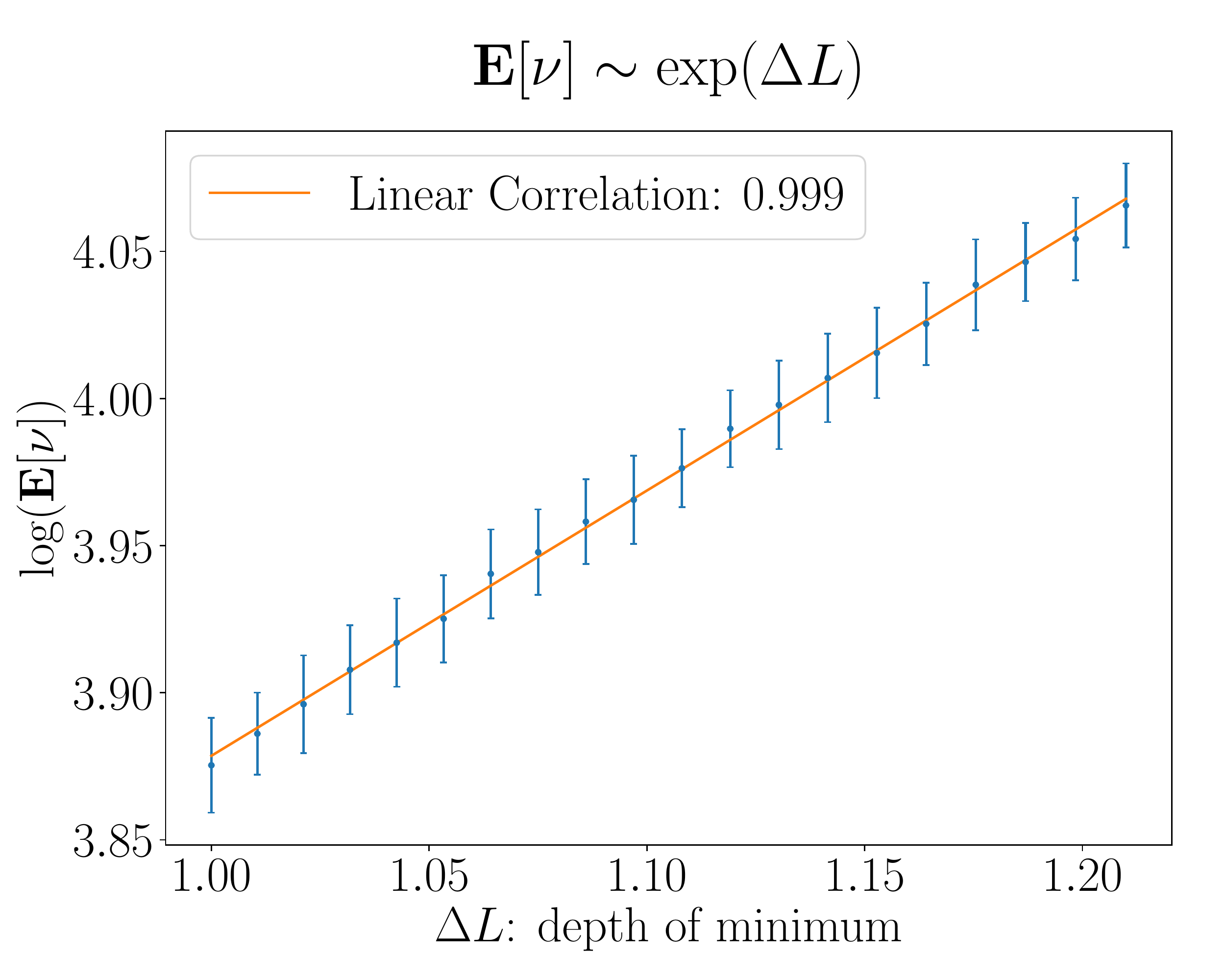}
\end{minipage}
\caption{
    Reference experiment using the proxy system (\ref{eq:sgld}).
    Different from the result of SGD, the exit time has no dependency on $\lambda_{\mathrm{max}}$
    while it has exponential dependency on $\eta$ and $\Delta L$ similarly to SGD.
    This result is aligned with Proposition \ref{thm:proxy_reuslt}.
}
\label{fig:reference}
\end{figure}

%% file: comparison.tex
\section{Comparison with Existing Escape Analyses}
\label{sec:comparison}
We compare our analysis with the closely related existing analyses
and discuss the technical differences in detail.
As summarized in Table \ref{table:result} and \ref{table:implication},
we picked as closely related analysis, \citep{Hu2017-ny, Jastrzebski2017-gm,Zhu2019-og,Nguyen2019-xf,Xie2020-ty},
which analyze how the SGD's noise affect escape efficiency.

\paragraph{Comparison on exit time}
From Table \ref{table:result}, we obtain three implications.
(i) In all the results, either or both the learning rate $\eta$ and $H^*$ play an important role.
(ii) There are four results where the exist time is expressed as an exponential form, and the sharpness-related values $\lambda_{\max}$ and $\bar{\lambda}$ appear in the results of \citet{Xie2020-ty} and our study.
(iii) Our study and \citet{Xie2020-ty} have different orders for the parameters for sharpness.
This fact will be discussed in the latter half of this section.

\paragraph{Escaping path assumption}
We remark that the assumptions of our theorem have an essential difference from \cite{Jastrzebski2017-gm} and \cite{Xie2020-ty}.
Their analyses assume that SGD escapes along a linear path, named ``escape path,''
where the gradient perpendicular to the path direction is zero.
Escaping path is a convenient assumption
to reduce the escape analysis to one-dimensional problems.
However, the existence of such paths is supported only weakly by \cite{draxler2018essentially},
and it is unlikely that the stochastic process continuously moves linearly.
The fact that we eliminated the escaping path assumptions
is a substantial technical improvement.

\paragraph{Effect of sharpness}
The technical significance of our theory is that
it can analyze the sharpness effect.
Because of its non-linearity,
sharpness analyses tend to become non-trivial,
thus a limited number of existing works have tackled it.
Among the selected results,
the sharpness effect appears in \cite{Jastrzebski2017-gm} and \cite{Zhu2019-og} as $H^*$,
and in \cite{Xie2020-ty} as $\lambda$.
We note that the results of \cite{Jastrzebski2017-gm} and \cite{Xie2020-ty}
include auxiliary sharpness values, such as $\bar{\lambda} \in [\lambda_{\min}, \lambda_{\max}]$ respectively.
Those terms appear because of the escaping path assumption
and our results show that those terms are not fundamental.

\paragraph{Heavy tailed noise}
Among the selected works, only \cite{Nguyen2019-xf} use a heavy-tailed noise model,
i.e. the noise whose distribution has a heavier tail than exponential distribution.
Although it is known that the heavy-tailed noise models the empirical behavior of SGD well \citep{Simsekli2019-et},
it is quite difficult to mathematically formulate it.
\cite{Nguyen2019-xf} use the Lévy process for their analysis,
where $\alpha$ represents the degree of the heavy tail,
and $\delta \in (0,1)$ includes miscellaneous constants.
Analyzing the sharpness under the heavy-tailed setup is still an open problem.
\begin{table}[ht]
\centering
\begin{tabular}{ c c}
\hline\hline\\[-1ex] 
 Study &  Exit Time (Order) \\ [1ex]
 \hline  \\ [-1ex] 
\centering
\cite{Hu2017-ny} & $\exp(\frac{1}{\eta})$  \\[1.5ex]
\cite{Jastrzebski2017-gm} & $\exp \left(\frac{B}{\eta }\Delta L + d\right)$ 
\\ [1.5ex]
\cite{Zhu2019-og} &  $\frac{1}{\operatorname{Tr}((H^{*})^2)}$
 \\[1.5ex]
\cite{Nguyen2019-xf}& $\frac{1}{{\eta}^{(\alpha - \delta)/2}}$ 
\\ [1.5ex]
\cite{Xie2020-ty} & $\exp \left(\frac{ B }{\eta}\Delta L\bar{\lambda}^{-1}\right)$ 
\\  [1.5ex]
Ours  &
$
  \exp\left(
        \frac{B}{\eta}\left(
                        \Delta L \lambda_\mathrm{max}^{-{1} / {2}}
                        \pm
                         \Xi
                      \right)
      \right)
$
\\[2.5ex]
\hline
\end{tabular}
\vspace{0.7cm}
\caption{Results of the derived exit time. We only show their order by ignoring constants. $\Xi = A \lambda_\mathrm{min}^{{1} / {2}}(\kappa^{{1} / {2}} - 1 )$ is the approximation error, $\bar{\lambda}$ is some value in $[\lambda_{\min}, \lambda_{\max}]$ defined in \citet{Xie2020-ty}, and $\alpha, \delta$ are parameters related to the tail probability \citep{Nguyen2019-xf}.}
\label{table:result}
\end{table}

\begin{table}[ht]
\centering
\begin{adjustwidth}{1.1cm}{}
\begin{tabular}{ c c c c c c }
\hline\hline\\ [-1ex] 
 \multirow{2}{*}{Studies} & Exponential       & Sharpness  &  No           & Non-        & \multirow{2}{*}{Discreteness}  \\ 
                          & escape            &  analysis    & escape paths  & stationary  &                               \\ [1ex]
 \hline  \\   [-1ex]
\citet{Hu2017-ny}          &       &       &$\surd$& $\surd$ & $\surd$  \\[1.5ex]
\citet{Jastrzebski2017-gm} &       &$\surd$&       &         &          \\[1.5ex]
\citet{Zhu2019-og}         &       &$\surd$&$\surd$& $\surd$ &          \\[1.5ex]
\citet{Nguyen2019-xf}      &       &       &$\surd$& $\surd$ & $\surd$  \\[1.5ex]
\citet{Xie2020-ty}         &$\surd$&$\surd$&       &         &          \\[1.5ex]
Ours                       &$\surd$&$\surd$&$\surd$& $\surd$ & $\surd$ \\[1.5ex]
\hline
\end{tabular}
\end{adjustwidth}
\vspace{0.5cm}
\caption{Technical difference among analyses.
The specific meanings of each column are described in the main passages of Section \ref{sec:comparison}.}
\label{table:implication}
\end{table}

%% file: related_works.tex
\begin{figure}[t]
\centering
\begin{minipage}[c]{0.325\textwidth}
\centering
    \includegraphics[width=\textwidth]{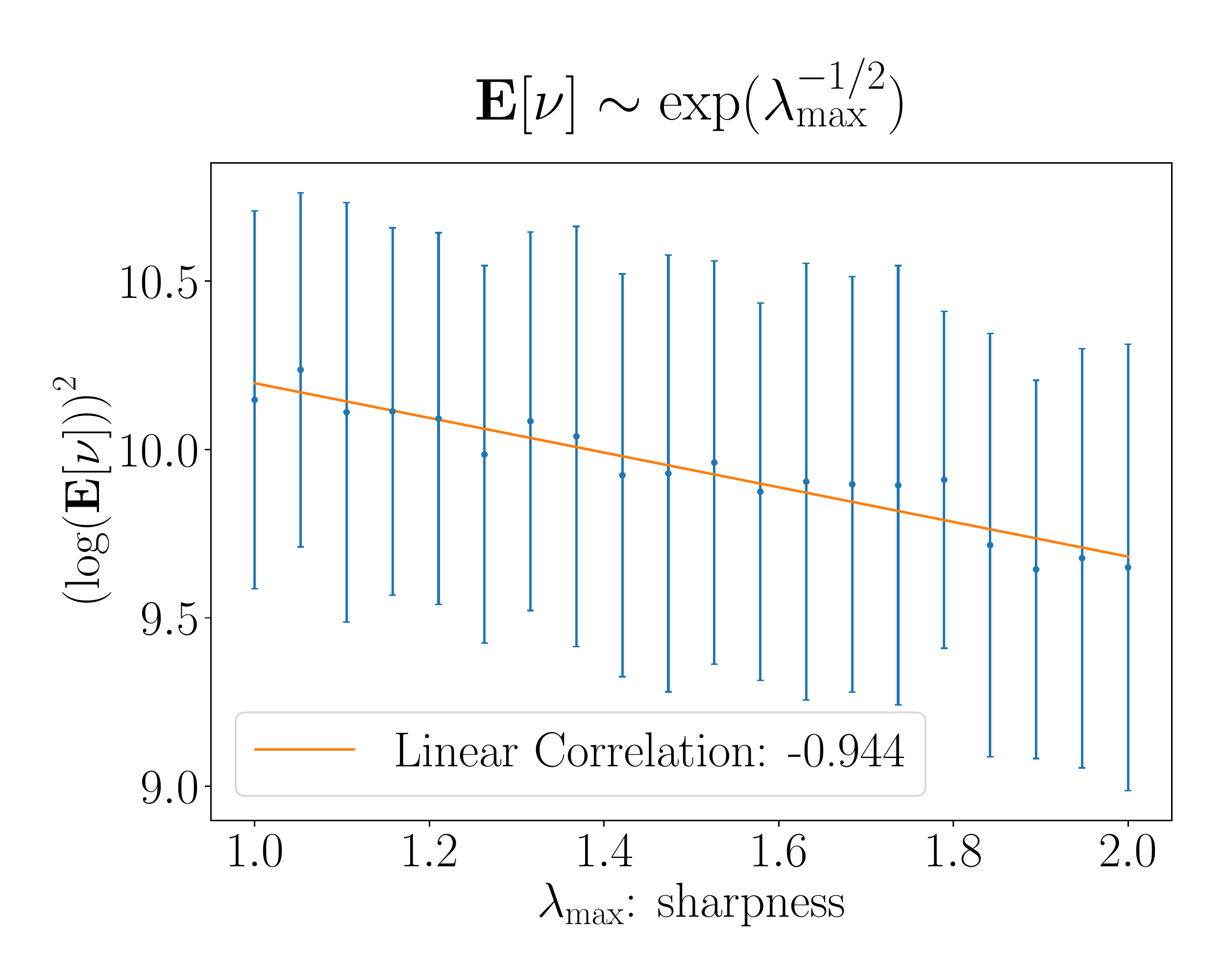}
\end{minipage}
\begin{minipage}[c]{0.325\textwidth}
\centering
    \includegraphics[width=\textwidth]{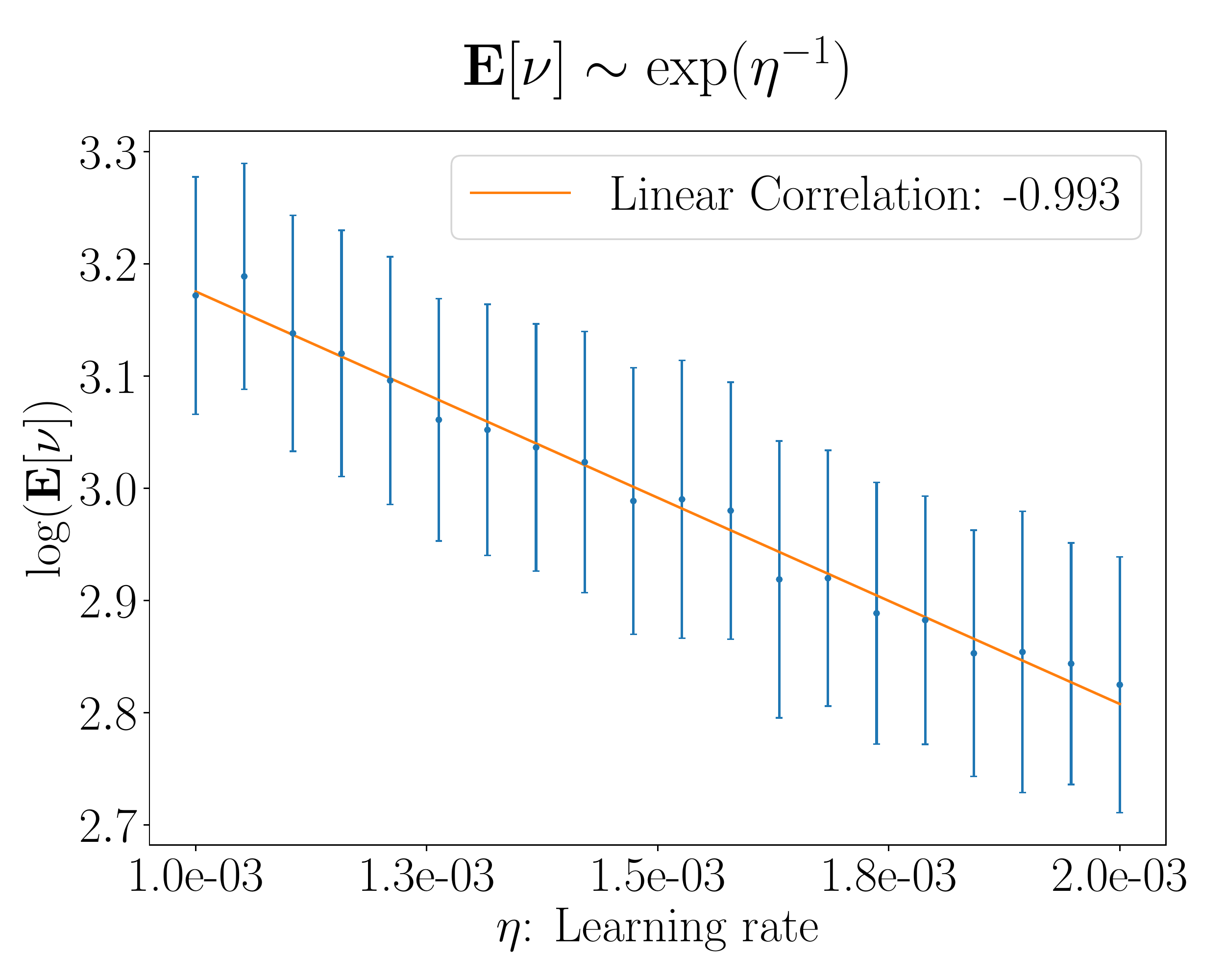}
\end{minipage}
\begin{minipage}[c]{0.325\textwidth}
\centering
    \includegraphics[width=\textwidth]{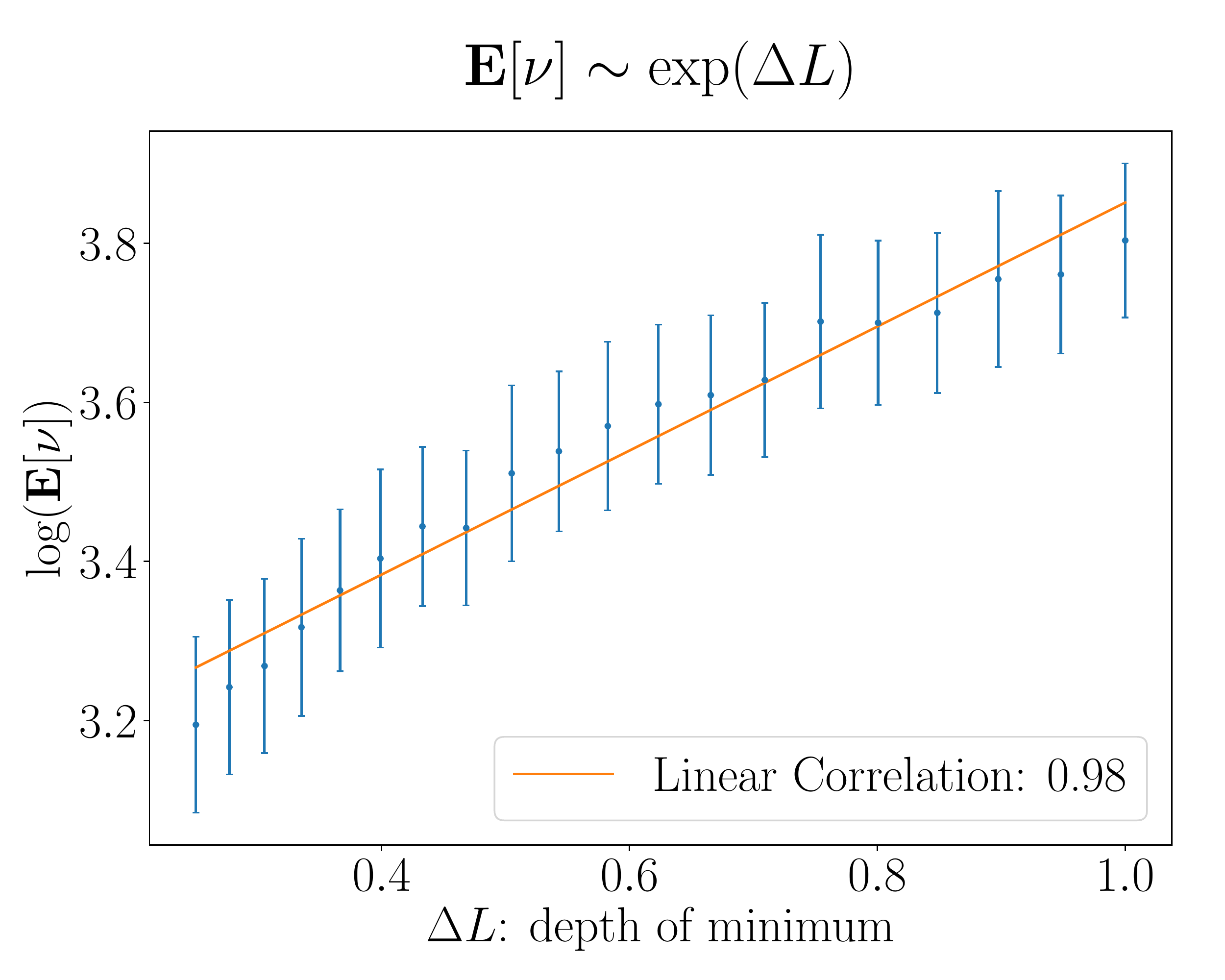}
\end{minipage}
\caption{A numerical experiment as in Section \ref{sec:experiment},
    but with discrete SGD (\ref{eq:gaussian_sgd}) rather than the original SGD (\ref{eq:original_sgd}).
    Although they are defined differently,
    they show similar trends on each parameter.
    This numerical result shows that (\ref{eq:gaussian_sgd}) is a 
    reasonable model for (\ref{eq:original_sgd}) for exit time analysis.
}
\end{figure} \label{fig:discrete_SGD}
\section{Related Work}
\paragraph{Sharpness and generalization of neural networks}
The shape of loss surfaces has long been a topic of interest.
The argument that the flatness of loss surfaces around local minima improves generalization was first studied by \cite{Hochreiter1995-vu, hochreiter1997flat}, and the observation has recently been reconfirmed in deep neural networks by \cite{Keskar2016-tn}.
The theoretical properties of the flatness were criticized by \cite{Dinh2017-km}
in terms of scale-sensitivity of flatness,
but there have also been follow-up works tackling the criticism by developing scale-invariant flatness measures \citep{tsuzuku2020normalized, Rangamani2019-hl, ibayashi2021minimum}.
Despite the ongoing theoretical controversy,
its empirical benefits seem to be promised \citep{Jiang2019-ci},
thus the training algorithms with sharpness regularizer
have achieved state-of-the-art \cite{Foret2020-kh,kwon2021asam}.
As the other investigations on the loss surface geometry,
\cite{He2019-mg} discussed the asymmetry of loss surfaces,
\cite{draxler2018essentially,garipov2018loss} studied how multiple local minima are internally connected,
and \cite{li2018visualizing} developed a random dimensional reduction method to visualize loss surfaces in low dimensions.

\paragraph{SGD and machine learning}
The detailed nature of SGD itself is also an object of interest.
SGD was first proposed in \citep{robbins1951stochastic},
as a lazy version of gradient descent using random subsets of training data.
Thus, SGD has been intended to be a convenient heuristic rather than a refined algorithm.
However besides its computational convenience,
SGD works as effectively as gradient descent does in many optimization problems,
and its convergence properties have been solidified on the convex objective functions \citep{bottou2010large}.
The recent success in the field of neural networks is particularly remarkable
because it is shown that SGD performs greatly on various non-convex functions as well.
In fact, SGD-based training algorithms have been achieving state-of-the-art one after another,
e.g., Adagrad \citep{duchi2011adaptive}, Adam \citep {kingma2014adam} and many others \citep{schmidt2021descending}.

\paragraph{Noise of SGD}
Analyzing SGD's noise has been an appealing topic in the research community.
It is known that the magnitude of the gradient noise in SGD has versatile effects on its dynamics \cite{kleinberg2018alternative}
thus it has been closely investigated especially in relation to a learning rate and a batch size.
An effect of large batch sizes on the reduction of gradient noise is
investigated in \cite{hoffer2017train,smith2018don,masters2018revisiting}.
Another area of interest is the shape of a gradient noise distribution.
\cite{Zhu2019-og, Hu2017-ny, Daneshmand2018-yf} investigated the anisotropic nature of gradient noise and its advantage.
\cite{Simsekli2019-et} discussed the fact that a gradient noise distribution has a heavier tail than Gaussian distributions. \cite{Nguyen2019-xf,csimcsekli2019heavy} showed benefits of these heavy tails for SGD. \cite{panigrahi2019non} rigorously examined gradient noise in deep learning and how close it is to a Gaussian. 
\cite{Xie2021-ty} studied a situation where the distribution is Gaussian, and then analyzes the behavior of SGD in a theoretical way.

\paragraph{Discretization of SGD}
We summarize the approximation we used in Table \ref{fig:sgd_models}.
We used the continuous SGD (\ref{def:continuous_sgd})
as an approximation of the discrete SGD (\ref{eq:gaussian_sgd})
because (\ref{def:continuous_sgd}) is exactly discretized to (\ref{eq:gaussian_sgd}).
This approximation is commonly used because it is well known that the trajectories of those two system show, so-called, ``strong convergence''
in the order of $O(\sqrt{\eta})$,
i.e.  $\mathbb{E}(\sup _{0 \leq t \leq T}|\theta_{k}^\mathrm{discrete}-\theta_{k\eta}|)=O(\eta^{\frac{1}{2}})$
(see e.g. \citep{Gobet2016-gz,Cheng2020-cp}).
We note that strong convergence validates the similarity of trajectories,
but it does not necessarily guarantee the similarity of escaping behavior.
Our work is the first completed argument with Lemma \ref{lemma:discretization} introduced.

As a final remark, (\ref{eq:gaussian_sgd}) is also an approximated model
of the original SGD (the dot arrow in Fig. \ref{fig:sgd_models}).
Although this approximation is justified via the central limit theorem \citep{Jastrzebski2017-gm,he2019control},
it is admittedly heuristic and
the quantitative validation for the approximation
is assumed to be a (highly non-trivial) open problem.
In Fig. \ref{fig:discrete_SGD}, we provide empirical results
to justify this approximation in our setup for completeness.

\section{Conclusion}
In this paper,
we showed that SGD has exponential escape efficiency from sharp minima
even in the non-stationary regime.
To reach the goal, we used the Large Deviation Theory
and identified that steepness plays the key role in the exponential escape
in the non-stationary regime.
Our results are the novel theoretical clue to explain the mechanics
as to why SGD can find generalizing minima.

%% file: stability_assumptions.tex
\section{Stability and Attraction of minima}
The followings are the minimum required assumptions for Theorem \ref{thm:basic_exit_time}
while both of them are derived from Assumption \ref{assumption:quadratic}.
\label{appendix:stability_attraction}
\begin{assumption}[$\theta^*$ is asymptotically stable]
  \label{assumption:stable}
  For any neighborhood $U$ that contains $\theta^*$, there exists a small neighborhood $V$ of $\theta^*$ such that gradient flow with any initial value $\theta_0 \in V$ does not leave $U$ for $t \geq 0$ and $\lim_{t \rightarrow \infty}\theta_{t} = \theta^*$.
\end{assumption}
\begin{assumption}[$D$ is attracted to $\theta^*$]
  \label{assumption:attracted}
  $\forall \theta_0 \in D$, a system $\dot{\theta}_t = -\nabla L(\theta_t)$ with initial value $\theta_0$ converges to $\theta^*$ without leaving $D$ as $t \rightarrow \infty$.
\end{assumption}
\textit{Stability} is a commonly used notion in  dynamical systems \citep{Hu2017-ny,Wu2017-ox}, although it does not always appear in SGD's escaping analysis \citep{Zhu2019-og, Jastrzebski2017-gm, Xie2020-ty}.
Assumption \ref{assumption:stable} is known to be equivalent to the local minimality of $\theta^*$
under the condition that $L(\theta)$ is real analytic around $\theta^*$ \citep{Absil2006-by}.
Also, by definition of asymptotic stability in Assumption \ref{assumption:stable}, we can always find a region $D$ that satisfies Assumption \ref{assumption:attracted}.
The more detailed properties of stability can be found, such as in Section 6.5 of \citet{Teschl2000-kj}.

%% file: mean_exit_time.tex
\section{Proof of Theorem \ref{thm:basic_exit_time}}
\label{appendix:basic_exit_time}
For simplicity, we use $\varepsilon$ to denote $\sqrt{{\eta}/{B}}$.
To prove this result, we provide the proof for an upper bound (Lemma \ref{lemma:exit_time_upper_bound}) and a lower bound (Lemma \ref{lemma:exit_time_lower_bound}).
Throughout the proofs, we use $\mathbf{C}_{T,\theta_0}$, $\mathrm{P}_{\theta_0}$, instead of $\mathbf{C}_{T}$ or $\mathrm{P}$, to clearly indicate which trajectory we are referring to.

We introduce several notions.
For $\delta > 0$ and $\theta \in \mathbb{R}^d$, let $\mathcal{B}_\delta(\theta)$ denote an $\delta$-neighbourhood of $\theta$, that is, $\mathcal{B}_\delta(\theta) := \{\theta' \in \mathbb{R}^d \mid \|\theta' - \theta\| \leq \delta\}$.
Further, for a set $\Theta \subset \mathbb{R}^d$, $\mathcal{B}_\delta(\Theta) := \cup_{\theta \in \Theta }\mathcal{B}_\delta(\theta)$.

The following lemma provides preliminary facts for proofs.
\begin{lemma}
    \label{lemma:auxiliary_trajectories}
    For any $c > 0$, there exist $\mu_1,\mu_2,T_1,T_2>0$ such that the followings hold:
    \begin{enumerate}
        \item $\forall \theta\in D$, there exists a trajectory $\varphi^1$ such that  $\varphi_0^1 = \theta$, $\varphi_T^1 \in \mathcal{B}_{\mu_1/2}(\theta^*)$, $0<T\leq T_1$ and $S_{T}(\varphi^1)=0$.
        \item $\forall \theta \in \mathcal{B}_{\mu_1}(\theta^*)$, there exists a trajectory $\varphi^2$ such that $\varphi^2_0 = \theta $, $\varphi^2_T \in \partial \mathcal{B}_{\mu_2}(D)$, $0<T\leq T_2$ and $S_{T}(\varphi^2)<V_{0}+\frac{c}{2}$.
    \end{enumerate}
\end{lemma}
\begin{figure}[ht]
    \centering
    \includegraphics[width=0.5\textwidth]{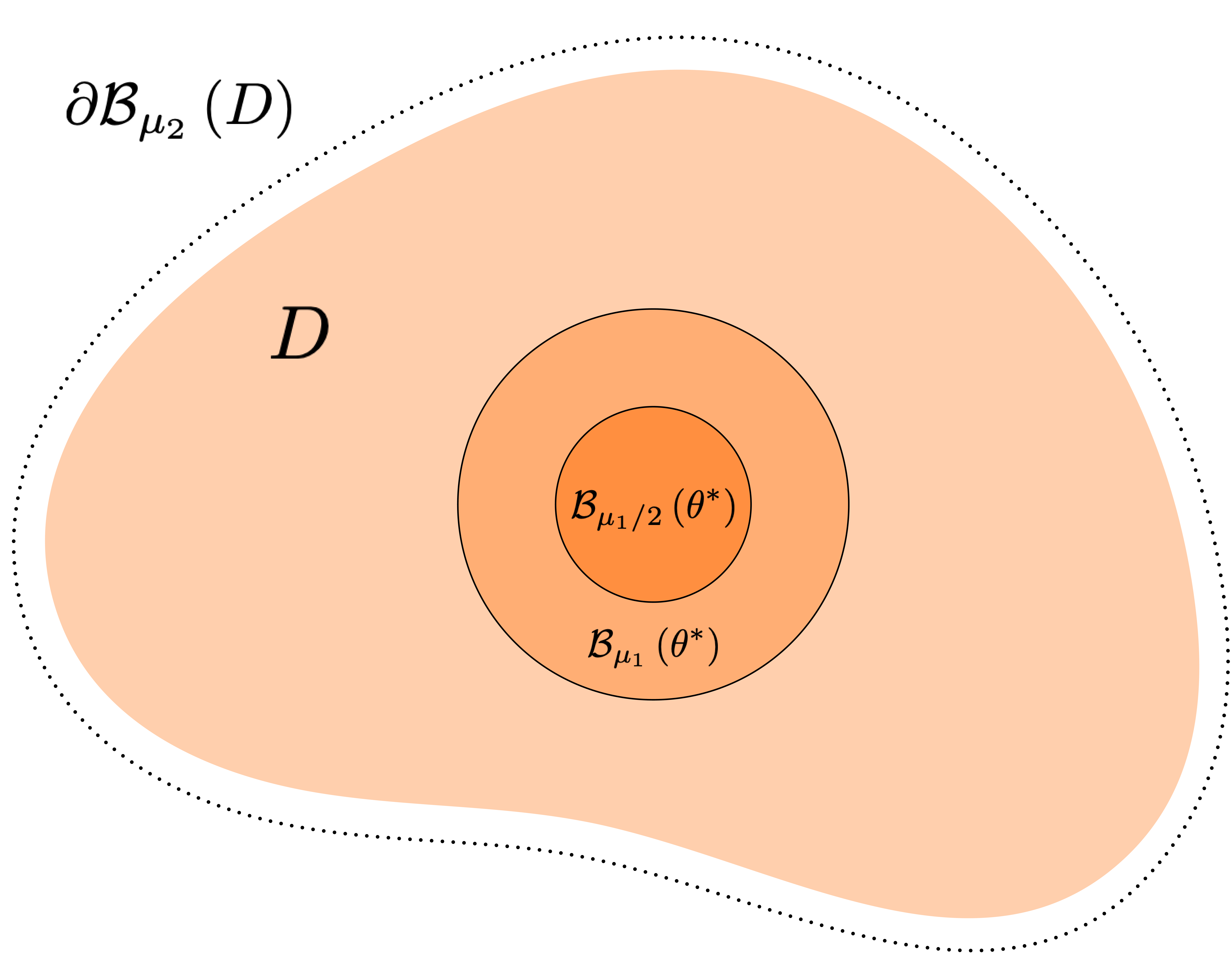}
    \caption{Illustration of domains and boundary, $D$, $\mathcal{B}_{\mu_1/2}(\theta^*)$, $\mathcal{B}_{\mu_1}(\theta^*)$, and $\partial \mathcal{B}_{\mu_2}(D)$}
    \label{fig:domains}
\end{figure}
The illustration can be found in Fig. \ref{fig:domains}.
\begin{proof}[Lemma \ref{lemma:auxiliary_trajectories}]
    The first statement immediately holds by the fact that $D$ is attracted to a asymptotically stable equilibrium position $\theta^*$  (Assumption \ref{assumption:stable} and \ref{assumption:attracted}).

    For the second statement, since $V_{0}:=\min _{\theta^{\prime} \in \partial D} V\left(\theta^{\prime}\right)$, there exists a trajectory $\varphi^a$ such that $\varphi^a_0 = \theta^*$, $\varphi^a_{T_a} \in \partial D $ and $S_{T_a}\left(\varphi^{a}\right) = V_0$, where $T_a$ is finite
    by Lemma 2.2 (a) in \cite{Freidlin2012-iz}.
    We cut off the first portion of $\varphi^a$ up until the first intersecting point with $\mathcal{B}_{\mu_1}(\theta^*)$ and define it as $\varphi^b$.
    This means $\varphi^b_0 \in \mathcal{B}_{\mu_1}(\theta^*)$, $\varphi^b_{T_b} \in \partial D $ and $S_{T_b}\left(\varphi^{b}\right) < V_0$ hold.
    By Lemma 2.3 in \cite{Freidlin2012-iz}, there exists a trajectory from $\varphi^b_{T_b}$ to a point $\theta_{\mu_2}$ in $\partial\mathcal{B}_{\mu_2}(D)$ such that the steepness is less than $K |\theta_{\mu_2} - \varphi^b_{T_b}|$  with a constant $K>0$.
    Then, if we take a small enough $\mu_2$, we can obtain $\varphi^c$ such  $\varphi^c_0 = \varphi^b_{T_b}$, $\varphi^c_{T_c} \in \partial \mathcal{B}_{\mu_2}(D) $ and $S_{T_c}\left(\varphi^{c}\right) < \frac{c}{2}$.
    By connecting $\varphi^b$ and $\varphi^c$, we obtain an appropriate $\varphi^{2}$.
\end{proof}

\input{proofs/exit_time/upper_bound/main}
\input{proofs/exit_time/lower_bound/main}

%% file: proofs/exit_time/upper_bound/main.tex
\begin{lemma}
\label{lemma:exit_time_upper_bound}
For any $c>0$, 
there exists an $\varepsilon_{0}$ such that
for $\varepsilon<\varepsilon_{0}$, 
$\varepsilon^{2} \ln \mathbb{E}\left[ \tau\right] < V_{0}+c$ holds,
where $V_0 := \min _{\theta^\prime \in \partial D} V(\theta^\prime)$.
\end{lemma}

\begin{proof}[Proof of Lemma \ref{lemma:exit_time_upper_bound}]
We split the dynamical system (\ref{def:continuous_sgd}) of our interest into the first half and the second half, $\{\theta_t^1\}_{t}$ and $\{\theta_t^2\}_{t}$.
$\{\theta_t^1\}_{t}$ starts with $\theta_0^1 = \theta_0 \in D$ and terminates when it first reaches $\mathcal{B}_{\mu_1/2}(\theta^*)$.
We define the terminating time of $\{\theta_t^1\}_{t}$ as $\tau_1 := \min\{t > 0 : \theta_t^1 \in \mathcal{B}_{\mu_1/2}(\theta^*)\}$.
On the other hand, $\{\theta_t^2\}_{t}$ starts with $\theta_0^2 = \theta_{\tau_1} \in \mathcal{B}_{\mu_1/2}(\theta^*)$ and terminates when it first reaches $\partial D$.
We define the terminating time of $\{\theta_t^2\}_{t}$ as $\tau_2 := \min\{t > 0 : \theta_t^1 \in \partial D\}$.
Clearly, the exit time $\tau = \tau_1 + \tau_2$.

Regarding $\tau_1$ and $\tau_2$, we show the following two independent facts with sufficiently small $\varepsilon > 0$.
\begin{enumerate}
    \item[Fact 1]:
    $\tau_1$ is no more than $T_1$ with probability at least $1/2$.
    \item[Fact 2]:
    $\tau_2$ is no more than $T_2$ with probability at least $\exp \left\{-\varepsilon^{-2}\left(V_{0}+c\right)\right\}$.
\end{enumerate}

\textbf{Fact 1:}
Given the trajectory $\varphi_1$ provided by Lemma \ref{lemma:auxiliary_trajectories}, Lemma \ref{lemma:action_functional_1} gives us that if $\varepsilon < \varepsilon_\mathrm{stp1}(\varphi_1,\mu_1/2,1)$, the following inequality holds
\begin{align*}
    \mathrm{P}_{\theta^1_0}\Big\{\varphi^\prime\in \mathbf{C}_{T_1,\theta_0^1}(\R^d)\mid \rho\left(\varphi^\prime, \varphi_1\right)<\mu_1/2\Big\} \geq \exp \left\{-\varepsilon^{-2}\right\}.
\end{align*}
Therefore, if we take $\varepsilon <\min\{\sqrt{{1} / {\ln 2}},\varepsilon_\mathrm{stp1}(\varphi_1,\mu_1/2,1)\}$, we have 
\begin{align*}
    \mathrm{P}_{\theta^1_0}\Big\{\varphi^\prime\in \mathbf{C}_{T_1,\theta_0^1}(\R^d)\mid \rho\left(\varphi^\prime, \varphi_1\right)<\mu_1/2\Big\} \geq\frac{1}{2}
\end{align*}
Since the event of $\{\varphi^\prime\in \mathbf{C}_{T_1,\theta_0^1}(\R^d)\mid \rho\left(\varphi^\prime, \varphi_1\right)<\mu_1/2\}$ means that $\{\theta_t^1\}_{t}$ reaches $\mathcal{B}_{\mu_1}(\theta^*)$ in no later than $T_1$, we obtain the following which provides Fact 1.
\begin{align}
    \label{eq:exit_time_part1}
    \mathrm{P}_{\theta^1_0}\left\{\tau_1<T_{1}\right\} \geq\frac{1}{2}.
\end{align}

\textbf{Fact 2:}
Given the trajectory $\varphi_2$ provided by Lemma \ref{lemma:auxiliary_trajectories}, Lemma \ref{lemma:action_functional_1} tells us that if $\varepsilon < \varepsilon_\mathrm{stp1}(\varphi_2,\mu_2,{c} / {2})$, the following inequality holds
\begin{align*}
    \mathrm{P}_{\theta^2_0}\Big\{\varphi^\prime\in \mathbf{C}_{T_2,\theta^2_{0}}(\R^d)\mid \rho\left(\varphi^\prime, \varphi_2\right)< \mu_2 \Big\} \geq \exp \left\{-\varepsilon^{-2}\left(S_{T_2}\left(\varphi_2\right)+\frac{c}{2}\right)\right\}.
\end{align*}
$\{\varphi^\prime\in \mathbf{C}_{T_2,\theta^2_{0}}(\R^d)\mid \rho\left(\varphi^\prime, \varphi_2\right)<\mu_2\}$ is the event that $\{\theta_t^2\}_t$ goes out of $D$ in no more than the time $T_2$.
Also, we know that $S_{T_2}\left(\varphi_2\right) < V_0 +\frac{c}{2}$ by Lemma \ref{lemma:auxiliary_trajectories}.
Hence, we can conclude the following for Fact 2:
\begin{align}
\label{eq:exit_time_part2}
\mathrm{P}_{\theta^2_0}\left\{\tau_2<T_{2}\right\}
&\geq \exp \left\{-\varepsilon^{-2}\left(V_{0}+c\right)\right\}.
\end{align}

Combining (\ref{eq:exit_time_part1}) and (\ref{eq:exit_time_part2}), we can obtain 
\begin{align}
\label{eq:exit_time_prob}
\mathrm{P}_{\theta_0}\left\{\tau<T_{1}+T_{2}\right\} 
\geq \frac{1}{2} \exp \left\{-\varepsilon^{-2}\left(V_{0}+c\right)\right\}.
\end{align}
Since this is a simple exponential distribution, we can obtain the following expectation:
\begin{align*}
\mathbb{E}\left[\tau\right]  \leq 2\left(T_{1}+T_{2}\right) \exp \left\{\varepsilon^{-2}\left(V_{0}+c\right)\right\}
\end{align*}
By setting 
\begin{align*}
    \varepsilon <\frac{1}{\sqrt{\ln 2(T_1+T_2)}}\min\left\{\sqrt{\frac{1}{\ln 2}},\varepsilon_\mathrm{stp1}(\varphi_1,\mu_1/2,1),\varepsilon_\mathrm{stp1}(\varphi_2,\mu_2,{c} / {2})\right\},
\end{align*}
we can get 
\begin{align*}
\mathbb{E} \left[\tau\right]  \leq \exp \left\{\varepsilon^{-2}\left(V_{0}+c\right)\right\}.
\end{align*}
Then, we obtain the statement.
\end{proof}

%% file: proofs/exit_time/lower_bound/main.tex
Next, we develop the lower bound on the exit time.

\begin{lemma}
\label{lemma:exit_time_lower_bound}
For any $c>0$, 
there exists an $\varepsilon_{0}$ such that
for $\varepsilon<\varepsilon_{0}$, 
$\varepsilon^{2} \ln \mathbb{E}\left[ \tau\right] > V_{0}-c$ holds,
where $V_0 := \min
_{\theta^\prime \in \partial D} V(\theta^\prime)$.
\end{lemma}
\begin{proof}[Lemma \ref{lemma:exit_time_lower_bound}]
We consider a specific case where the initial value of (\ref{def:continuous_sgd}) is in $\partial\mathcal{B}_{\mu_1/2}(\theta^*)$, which can be trivially extended to general cases.
Consider a Markov chain $Z_k \;(k\in \mathbb{N})$ as a discretization of $\theta_t$ as $t = \tau_k$ with a $k$-th time grid $\tau_k$.
It is formally defined as follows:
\begin{enumerate}
    \item $\tau_0 = 0$,
    \item $\sigma_{k}=\inf \left\{t>\tau_{k}\right. \text { : }\left.\theta_{t} \in \partial\mathcal{B}_{\mu_1}(\theta^*)\right\}$,
    \item $\tau_{k}=\inf \left\{t>\sigma_{k-1}: \theta_{t} \in \partial\mathcal{B}_{\mu_1/2}(\theta^*) \cup \partial D\right\}$,
    \item $Z_{k}=\theta_{\tau_{k}}$.
\end{enumerate}
\begin{figure}[ht]
    \centering
    \includegraphics[width=0.5\textwidth]{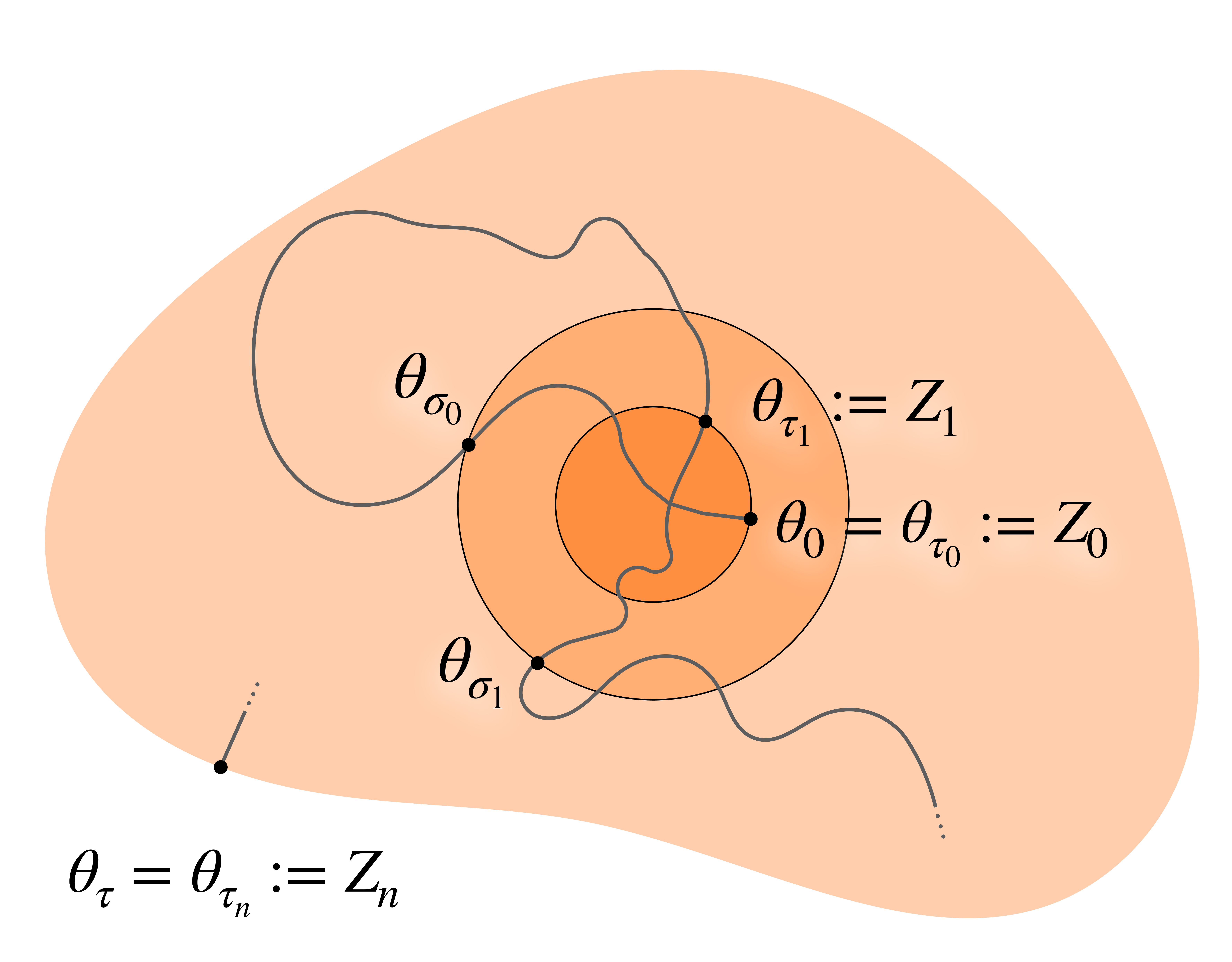}
    \caption{A continuous trajectory $\theta_t$ and the Markov chain $Z_k$ generated from $\{\theta_t\}_t$. Colored domains indicate $D$, $\mathcal{B}_{\mu_1/2}(\theta^*)$, and $\mathcal{B}_{\mu_1}(\theta^*)$ as illustrated in Fig. \ref{fig:domains}.}
    \label{fig:markov}
\end{figure}
By introducing $Z_k$, we can reduce the continuous process $\{\theta_t\}_t$ to a discrete Markov chain transiting between $\partial \mathcal{B}_{\mu_{1}/2}\left(\theta^{*}\right)$ and $\partial D$.
The illustration can be found in Fig. \ref{fig:markov}.

Let $\kappa := \inf\{k \mid Z_k  \in \partial D\}$.
Then, we have $\tau = \tau_n$ and 
\begin{align*}
    \mathbb{E} \left[\tau\right] &= \sum_{k=0}^{\infty}\Big( \mathrm{P}_{\theta_0}\left\{\kappa \geq k\right\} - \mathrm{P}_{\theta_0}\left\{\kappa \geq k+1\right\}\Big) \tau_{k}= \sum_{k=1}^{\infty} \mathrm{P}_{\theta_0}\left\{\kappa \geq n\right\} \left(\tau_{k}-\tau_{k-1}\right).
\end{align*}
This can be further evaluated as
\begin{align*}
\mathbb{E} \left[\tau\right]
&> \sum_{k=1}^{\infty} \mathrm{P}_{\theta_0}\left\{\kappa \geq k\right\} \left(\tau_{k}-\sigma_{k-1}\right)> \sum_{k=1}^{\infty} \mathrm{P}_{\theta_0}\left\{\kappa \geq n\right\}\Big( \inf _{\theta_0 \in \partial \mathcal{B}_{\mu_1}(\theta^*)}\mathbb{E} [\tau_1]\Big),
\end{align*}
which follows $\tau_{k-1} < \sigma_{k-1}$.
Since $\mathcal{B}_{\mu_1/2}(\theta^*)$ is a strict subset of $\mathcal{B}_{\mu_1}(\theta^*)$, and $\mathcal{B}_{\mu_1}(\theta^*)$ is a strict subset of $D$, it takes a positive amount of time to transit from $\partial\mathcal{B}_{\mu_1}(\theta^*)$ to either $\partial\mathcal{B}_{\mu_1/2}(\theta^*)$ or $\partial D$, and there exists a positive lower bound $t_1$ for $\inf _{\theta_0 \in \mathcal{B}_{\mu_1}(\theta^*)}\mathbb{E}[ \tau_1] $ that is independent of $\varepsilon$.
Thus we get 
\begin{align*}
\mathbb{E} \left[\tau\right] >  t_1\sum_{k=1}^{\infty} \mathrm{P}_{\theta_0}\left\{\kappa \geq k\right\}.
\end{align*}
By Lemma \ref{lemma:trasit_prob}, we immediately get $\mathrm{P}_{\theta_0}\{\kappa>k\} \geq[1-\exp \{-\varepsilon^{-2}(V_{0}-c)\}]^{k-1}$, hence we have
\begin{align*}
\mathbb{E} \left[\tau\right] &>   t_1 \sum_{k=1}^{\infty} \left[1-\exp \left\{-\varepsilon^{-2}\left(V_{0}-c\right)\right\}\right]^{k-1}= t_1 \exp \left\{\varepsilon^{-2}\left(V_{0}-c\right)\right\}.
\end{align*}
This implies $\mathbb{E}[\tau] \geq \exp \{\varepsilon^{-2}(V_{0}-c)\}$ holds if $\varepsilon$ is small enough.
\end{proof}

\begin{lemma}
\label{lemma:trasit_prob}
We obtain
$$
\mathrm{P}(Z_{k+1}\in \partial D \mid Z_k\in \mathcal{B}_{\mu_1/2}(\theta^*)) \leq \exp \left\{-\varepsilon^{-2}\left(V_{0}-c\right)\right\}.
$$
\end{lemma}
\begin{proof}[Lemma \ref{lemma:trasit_prob}]
First, we decompose $\mathrm{P}(Z_{k+1}\in \partial D \mid Z_k\in \partial \mathcal{B}_{\mu_1/2}(\theta^*))$ into two parts in the following way:
\begin{align}
\label{eq:lower_bound_transit_prob}
\begin{aligned}
&\mathrm{P}_{\theta_0}(Z_{k+1}\in \partial D \mid Z_k\in \partial\mathcal{B}_{\mu_1/2}(\theta^*)) \\
& \leq \max _{\theta'_0 \in \partial\mathcal{B}_{\mu_1/2}(\theta^*) } P_{\theta'_0}\left\{\tau_{1} =\tau\right\} \\
&=\max _{\theta'_0 \in \partial\mathcal{B}_{\mu_1/2}(\theta^*)}\Big[\mathrm{P}_{\theta'_0}\left\{\tau=\tau_{1}<T\right\}+\mathrm{P}_{\theta'_0}\left\{\tau=\tau_{1} \geq T\right\}\Big]
\end{aligned}
\end{align}
This holds for arbitrary $T$, so we pick $T=T^\prime$ large enough so that this inequality holds for the first term:
\begin{align}
    \label{eq:lower_bound_first_term}
    \mathrm{P}_{\theta'_0}\left\{\tau=\tau_{1} \geq T^\prime\right\} \leq \frac{1}{2} \exp \left\{-\varepsilon^{-2}\left(V_{0}-c\right)\right\}
\end{align}
The existence of such $T^\prime$ is guaranteed by the fact that $V_0$ is finite and the following lemma.
\begin{lemma}[Lemma 2.2 (b) in \cite{Freidlin2012-iz}]
For any $\alpha>0$,
there exists positive constants $c$ and $T_{0}$, such that for all sufficiently small $\varepsilon > 0$ and any $\theta_0 \in D \cup \partial D \backslash  \mathcal{B}_{\alpha}(\theta^*)$  we have the inequality
$$
\mathrm{P}_{\theta_0}\left\{\zeta_{\alpha}>T\right\} \leq \exp \left\{-\varepsilon^{-2} c\left(T-T_{0}\right)\right\},
$$
where $\zeta_{\alpha}=\inf \left\{t: \theta_{t} \notin D \backslash \mathcal{B}_{\alpha}(\theta^*)\right\}$.
\end{lemma}
Given a constant $T^\prime$, we consider bounding $\mathrm{P}_{\theta'_0}\left\{\tau=\tau_{1}<T^\prime\right\}$.
Consider the following set of trajectories:
\begin{align*}
    \Phi(V_0 - c/2):=\{\varphi: \varphi_0 = \theta'_0, S_{T}(\varphi) \leq V_0 - c/2\}.
\end{align*}
Since it takes at least $V_0$ to reach $\partial D$ from $\theta^*$, the following inequality holds:
\begin{align*}
\mathrm{P}_{\theta'_0}\left\{\tau=\tau_{1}<T^\prime\right\}&\leq \mathrm{P}_{\theta'_0}\left\{\varphi_y \notin\Phi(V_0 - c/2)\right\}.
\end{align*}
Also, Lemma \ref{lemma:action_functional_2} implies, for all $\varepsilon\leq  \varepsilon_\mathrm{stp2}(V_0 - c/2,\delta,c/2)$
\begin{align*}
\mathrm{P}_{\varphi^\prime}\Big\{\varphi^\prime\in \mathbf{C}_{T,\theta'_0}(\R^d)\mid\rho\big((\varphi^\prime-\Phi(V_0 - c/2)\big)\geq \delta\Big\}&\leq \exp \{-\varepsilon^{-2}((V_0 - c/2)-c/2)\} \\
&= \exp \{-\varepsilon^{-2}(V_0 - c)\}
\end{align*}
Since $\delta$ can be arbitrarily small, the event of $\{\varphi_{y} \notin \Phi(V_{0}-c / 2)\}$ is equal to the event of $\{\varphi^\prime\in \mathbf{C}_{T,\theta'_0}(\R^d)\mid\rho((\varphi^\prime-\Phi(V_0 - c/2))\geq \delta\}$.
Hence, we obtain
\begin{align}
    \mathrm{P}_{\theta'_0}\left\{\tau=\tau_{1}<T\right\} < \exp \{-\varepsilon^{-2}(V_0 - c)\}.
\end{align}
If we set $\varepsilon\leq  \frac{1}{\sqrt{\ln 2}}\varepsilon_\mathrm{stp2}(V_0 - c/2,\delta,c/2)$, we conclude
\begin{align}
    \label{eq:lower_bound_second_term}
    \mathrm{P}_{\theta'_0}\left\{\tau=\tau_{1}<T\right\} < \frac{1}{2}\exp \{-\varepsilon^{-2}(V_0 - c)\}.
\end{align}

Combining (\ref{eq:lower_bound_transit_prob}), (\ref{eq:lower_bound_first_term}), and (\ref{eq:lower_bound_second_term}), we prove the statement.
\end{proof}

%% file: hamilton_jacobi_quation.tex
\section{Hamilton-Jacobi Equation for Quasi-potential}
\label{appendix:quasi_potential}
While we use a proximal system to approximate the quasi-potential,
there have been attempts to directly analyze $V(\theta)$.
A prominent result is the theorem by \cite{Hu2019-vz},
which showed that $V(\theta)$ satisfies
the following Hamilton–Jacobi equation. 
\begin{theorem}
    \label{thm:jacobi_equation}
    For all $\theta\in D$, $V(\theta)$ satisfies the following Jacobi equation,
    $$
        \frac{1}{2}\nabla V(\theta)^{\top} {C\left(\theta\right )}^{1/2} \nabla V(\theta)-\nabla L\left(\theta\right)^\top \nabla V(\theta)=0
    $$
\end{theorem}
Although it does not give us a closed solution of $V(\theta)$, 
it reflects the role of $C(\theta)$ to make $V(\theta)$ smaller.
Below, we provide the proof in our notation for the completeness.
\input{proofs/quasi_potential}

%% file: proofs/quasi_potential.tex
\begin{proof}
For $u,v \in \mathbb{R}^d$, we introduce an inner product and a norm regarding a point $\theta\in D$ as
$\langle u, v\rangle_{\theta} := u^{\top} C\left(\theta\right )^{-1/2} v$ and 
$\|u\|_{\theta} := \sqrt{\langle u, u\rangle_{\theta}}$.
With these definitions, the $S_{T}(\varphi)$ is written as follows:
\begin{align}
\label{eq:simple_ratefunc}
S_{T}(\varphi)
= \frac{1}{2} \int_{0}^{T}\left\|\dot{\varphi}_t+\nabla L\left(\varphi_t\right)\right\|_{\varphi_t}^{2}   d t.
\end{align}
Note that $\varphi_t \in D$ holds for any $t \in [0,T]$ by the definition of trajectories.
We rewrite the integrand of (\ref{eq:simple_ratefunc}) as follows:
\begin{align} 
&\left\|\dot{\varphi}_t+\nabla L\left(\varphi_t\right)\right\|_{\varphi_t}^{2} \notag  \\
&=\left\|\dot{\varphi}_t\right\|_{\varphi_t}^{2}+\left\|\nabla L\left(\varphi_t\right)\right\|_{\varphi_t}^{2} +2\left\langle\dot{\varphi}_t, \nabla L\left(\varphi_t\right)\right\rangle_{\varphi_t} \notag \\
&=\left(\left\|\dot{\varphi}_t\right\|_{\varphi_t}-\left\|\nabla L\left(\varphi_t\right)\right\|_{\varphi_t}\right)^2 +2\left\|\dot{\varphi}_t\right\|_{\varphi_t}\left\|\nabla L\left(\varphi_t\right)\right\|_{\varphi_t}+2\left\langle\dot{\varphi}_t, \nabla L\left(\varphi_t\right)\right\rangle_{\varphi_t} \notag \\
\label{eq:rate_function_simple_ineqality}
&\geq 2\left\|\dot{\varphi}_t\right\|_{\varphi_t}\left\|\nabla L\left(\varphi_t\right)\right\|_{\varphi_t}
+2\left\langle\dot{\varphi}_t, \nabla L\left(\varphi_t\right)\right\rangle_{\varphi_t}.
\end{align}
We develop a parameterization for the term in (\ref{eq:rate_function_simple_ineqality}).
For a trajectory $\varphi$, we select an bijective function $f: [0,1] \to [0,T]$ as satisfying the follows:
for each $t \in [0,T]$ and $t^* \in [0,1]$ as $t=f(t^*)$, a parameterized trajectory $\varphi^*_{t^*}:= \varphi_{f(t^*)}$ satisfies
\begin{align}
\left\|\dot{\varphi}^*_{t^*}\right\|_{\varphi^*_{t^*}}=\|\nabla L(\dot{\varphi}^*_{t^*})\|_{\varphi^*_{t^*}}. \label{eq:constraint}
\end{align}
This parameterization reduces the quasi-potential to the minimum of the following quantity:
\begin{align}
\label{eq:minimum_ratefunction}
(S_{T}(\varphi) \geq S_{ f(T)}(\varphi^*)=)\int^{f(T)}_0\left\|\dot{\varphi}^*_{t^*}\right\|_{\varphi^*_{t^*}}\left\|\nabla L\left(\varphi^*_{t^*}\right)\right\|_{\varphi^*_{t^*}}
+\left\langle\dot{\varphi}^*_{t^*}, \nabla L\left(\varphi^*_{t^*}\right)\right\rangle_{\varphi^*_{t^*}} dt^*
\end{align}
subject to the constraint (\ref{eq:constraint}).
Since the integrand of (\ref{eq:minimum_ratefunction}) includes the first order derivative regarding $t^*$, (\ref{eq:minimum_ratefunction}) holds over different parameterizations $f$.
For convenience, we use another bijective parameterization function $g:[0,R] \to [0,1]$ as $t^*=g(r)$ with $R > 0$ and $r \in [0,R]$ such that $\hat{\varphi}_r := \varphi^*_{g(r)}$ satisfies
\begin{align}
    \label{eq:parameterizing}
    \|\dot{\hat{\varphi}}_r\|_{\hat{\varphi}_r} = 1.
\end{align}
Then, the quasi-potential is reduced to the following formula, \footnote{One might think that if we parametrize as above (\ref{eq:parameterizing}), the equality condition for (\ref{eq:rate_function_simple_ineqality}) is violated. Indeed 
$$
S_{T}(\hat{\varphi}) \neq \int^{T}_0\|\dot{\varphi}_t\|_{\hat{\varphi}_t}\left\|\nabla L\left(\hat{\varphi}_t\right)\right\|_{\hat{\varphi}_t}
+\left\langle\dot{\varphi}_t, \nabla L\left(\hat{\varphi}_t\right)\right\rangle_{\hat{\varphi}_t} dt
$$
for $\hat{\varphi}$.
However, $\hat{\varphi}$ is introduced just for the simple calculation of $S_{T}(\varphi^*)$.
Although $\hat{\varphi}$ frequently appears in the proof, our attention is still on $\varphi^*$ and $S_{T}(\varphi^*)$, not on $S_{T}(\hat{\varphi})$.}
\begin{align}
V(\theta)=\inf_{ r \in [0,R]:\|\dot{\hat{\varphi}}_r\|_{\hat{\varphi}_r}=1}
\int^{R}_0\left(\|\dot{\hat{\varphi}}_r\|_{\hat{\varphi}_r}\left\|\nabla L\left(\hat{\varphi}_r\right)\right\|_{\hat{\varphi}_r}
+\left\langle\dot{\hat{\varphi}}_r, \nabla L\left(\hat{\varphi}_r\right)\right\rangle_{\hat{\varphi}_r} \right)dr, \label{eq:quasi-pot}
\end{align}
where $\hat{\varphi}_R = \theta$.
By the Bellman equation-type optimality, we expand the right hand side of (\ref{eq:quasi-pot}) into the following form:
\begin{align}
\label{eq:bellman_optimality}
V(\theta)=\inf_{ r \in [0,R]:\|\dot{\hat{\varphi}}_r\|_{\hat{\varphi}_r}=1 }
\left(\int^{R}_{R-\delta}\left(\|\dot{\hat{\varphi}}_r\|_{\hat{\varphi}_r}\left\|\nabla L\left(\hat{\varphi}_r\right)\right\|_{\hat{\varphi}_r}
+\left\langle\dot{\hat{\varphi}}_r, \nabla L\left(\hat{\varphi}_r\right)\right\rangle_{\hat{\varphi}_r} \right) dr
+V(\hat{\varphi}_{R-\delta})\right),
\end{align}
with a width value $\delta > 0$.
The Taylor expansion around $r = R$ gives
\begin{align*}
&\int^R_{R-\delta}\left(\|\dot{\hat{\varphi}}_r\|_{\hat{\varphi}_r}\left\|\nabla L\left(\hat{\varphi}_r\right)\right\|_{\hat{\varphi}_r}
+\left\langle\dot{\hat{\varphi}}_r, \nabla L\left(\hat{\varphi}_r\right)\right\rangle_{\hat{\varphi}_r} \right)dr
+V(\hat{\varphi}_{R-\delta})\\
&=\delta\left(\|\nabla L(\hat{\varphi}_R)\|_{\hat{\varphi}_R}
+\left\langle\nabla L(\hat{\varphi}_R), \dot{\hat{\varphi}}_R\right\rangle_{\hat{\varphi}_R}-\dot{\hat{\varphi}}_R^\top\nabla V\left(\hat{\varphi}_R\right)\right) 
+V\left(\hat{\varphi}_R\right)+O\left(\delta^{2}\right)
\end{align*}
Taking $\delta\to 0$ and noticing $\hat{\varphi}_R = \theta$, (\ref{eq:bellman_optimality}) can be simplified to the following equation:
\begin{align}
\label{eq:mimiization_simple_rate_function}
0=\inf_{ r \in [0,R]:\|\dot{\hat{\varphi}}_r\|_{\hat{\varphi}_r}=1}
\left(\|\nabla L(\theta)\|_{\theta}
+\left\langle\nabla L(x)^\top, \dot{\hat{\varphi}}_R\right\rangle_{\theta}-\dot{\hat{\varphi}}_R\nabla V\left(\theta\right)\right)
\end{align}

It remains to select $\hat{\varphi}$ which solves the minimization problem (\ref{eq:mimiization_simple_rate_function}).
Since the following equality holds,
\begin{align}
\label{eq:part_rate_function}
\left\langle\nabla L(\theta)^\top, \dot{\hat{\varphi}}_R\right\rangle_{\theta}-\dot{\hat{\varphi}}_R\nabla V\left(\theta\right) = \left\langle\nabla L(\theta)^\top-\nabla V\left(\theta\right)^\top C\left(\theta\right )^{-1/2}, \dot{\hat{\varphi}}_R\right\rangle_\theta,
\end{align}
it is easy to see that a trajectory $\hat{\varphi}^*$ such that
\begin{align*}
\dot{\hat{\varphi}}^*_R= -\frac{\nabla L(x)^\top   -\nabla V\left(\theta\right)C\left(\theta\right )^{1/2}}{\|\nabla L(\theta)^\top   -\nabla V\left(\theta\right)C\left(\theta\right )^{1/2}\|_{\theta}}
\end{align*}
minimizes (\ref{eq:part_rate_function}).
With this $\hat{\varphi}^*$, (\ref{eq:mimiization_simple_rate_function}) simplifies to 
\begin{align*}
\|\nabla L(\theta)\|_{\theta}= \|\nabla L(\theta)^\top   -\nabla V\left(\theta\right)C\left(\theta\right )^{1/2}\|_{\theta}.
\end{align*}
Taking the square of both sides, we get the statement.
\end{proof}

\section{Proof of Lemma \ref{lemma:quasi-potential_estimate}}

\begin{proof}
First, we use $\lambda_\mathrm{max}^{-\frac{1}{2}}\hat{S}_{T}(\varphi)$ as a ``proxy steepness''
to estimate $S(\varphi)$ and $V(\theta)$.
For any trajectory $\varphi$, the following bound holds.
\begin{align}
  &\left|S_{T}(\varphi) - \lambda_\mathrm{max}^{-\frac{1}{2}}\hat{S}_{T}(\varphi)\right| \\
  &= \left|
      \frac{1}{2}\int_{0}^{T}\left(\dot{\varphi}_t+\nabla L\left(\varphi_t\right)\right)^{\top} C\left(\varphi_t\right )^{-1/2}\left(\dot{\varphi}_t+\nabla L\left(\varphi_t\right)\right) d t
      - \frac{1}{2}\lambda_\mathrm{max}^{-\frac{1}{2}}\int_{0}^{T}\left\|\dot{\varphi}_t+\nabla L\left(\varphi_t\right)\right\|^2 dt
     \right|\\
  & = \left|
        \frac{1}{2}\int_{0}^{T}
          \left(\dot{\varphi}_t+\nabla L\left(\varphi_t\right)\right)^{\top}
          \left(C\left(\varphi_t\right )^{-1/2} - \lambda_\mathrm{max}^{-\frac{1}{2}}I\right)
          \left(\dot{\varphi}_t+\nabla L\left(\varphi_t\right)\right)
        dt
      \right|\\
  & \leq \frac{1}{2}
         \int_{0}^{T}
            \left|
              \left(\dot{\varphi}_t+\nabla L\left(\varphi_t\right)\right)^{\top}
              \left(C\left(\varphi_t\right )^{-1/2} - \lambda_\mathrm{max}^{-\frac{1}{2}}I\right)
              \left(\dot{\varphi}_t+\nabla L\left(\varphi_t\right)\right)
            \right|
          dt
  \label{eq:steepness_upperbound}
\end{align}
Since $C\left(\varphi_t\right )^{-1/2} - \lambda_\mathrm{max}^{-\frac{1}{2}}I$ is positive semi-definite,
\begin{align}
   (\ref{eq:steepness_upperbound})
    & \leq \frac{1}{2}\int_{0}^{T}
      \left\|
        \dot{\varphi}_t+\nabla L\left(\varphi_t\right)
      \right\|^2
      \lambda_\mathrm{max}\left(
        C\left(\varphi_t\right )^{-1/2} - \lambda_\mathrm{max}^{-\frac{1}{2}}I
      \right)
    dt
   \label{eq:steepness_upperbound_psd}
\end{align}
Since $D$ is a finite set and $L(\theta)$ is a locally quadratic funciton (Assumption \ref{assumption:quadratic}),
there exists a constant $M>0$ that satisfies $\forall \theta\in D: \|\nabla L(\theta)\| \leq M$.
Combined with Assumption \ref{assumption:strong_covariance} and \ref{assumption:trajectory_grad_bound},
we can further obtain the following bound.
\begin{align}
  (\ref{eq:steepness_upperbound_psd})
  & \leq \frac{T}{2}(K+M)^2
         \sup_{0\leq t \leq T}
            \left\{
              \lambda_\mathrm{max}\left(C\left(\varphi_t\right )^{-1/2}
              -
              \lambda_\mathrm{max}^{-\frac{1}{2}}I\right)
            \right\} \\
  &\quad(\because \|\nabla L(\theta)\| \leq M \textrm{ and Assumption \ref{assumption:trajectory_grad_bound}}) \nonumber\\
  & = \frac{T}{2}(K+M)^2
         \sup_{0\leq t \leq T}
           \left\{
             \lambda_\mathrm{max}\left(C\left(\varphi_t\right )^{-1/2}\right)
             -
             \lambda_\mathrm{max}^{-\frac{1}{2}}
           \right\} \\
  & \leq \frac{T}{2}(K+M)^2
        \left(\lambda_\mathrm{min}^{-\frac{1}{2}} - \lambda_\mathrm{max}^{-\frac{1}{2}}\right)
        \quad(\because \textrm{Assumption \ref{assumption:strong_covariance}})
\end{align}
With this upper bound,
  $V_0$ can also be bounded in the followings.
By definition,
\begin{align}
  V_0 
    &= \inf_{\theta\in\partial D}\inf_{T>0}\inf_{\varphi: \substack{(\varphi_{0}, \varphi_T)=(\theta^*, \theta)}} S_{T}(\varphi) \\
 \hat{V}_0
    &= \inf_{\theta\in\partial D}\inf_{T>0}\inf_{\varphi: \substack{(\varphi_{0}, \varphi_T)=(\theta^*, \theta)}}
        \hat{S}_{T}(\varphi)
\end{align}
From here below, we denote $\inf_{\varphi: \substack{(\varphi_{0}, \varphi_T)=(\theta^*, \theta)}}$
by $\inf_{\varphi(\theta,T)}$ for brevity.

Since $\partial D$ is a continuous finite boundary,
we have the following $\theta^\dagger$  and $\theta^*$.
\begin{align}
  \theta^\dagger
    := \underset{\theta\in\partial D}{\arg\!\inf}\inf_{T>0} \inf_{\varphi(\theta,T)} S_{T}(\varphi) \\
  \theta^*
    := \underset{\theta\in\partial D}{\arg\!\inf}\inf_{T>0} \inf_{\varphi(\theta,T)} \hat{S}_{T}(\varphi).
\end{align}
The followings hold.
\begin{align}
  \inf_{\theta\in\partial D}
  \inf_{T>0}
  \inf_{\varphi(\theta,T)} S_{T}(\varphi)
- \inf_{\theta\in\partial D}
  \inf_{T>0}
  \inf_{\varphi(\theta,T)}
    \hat{S}_{T}(\varphi)
  \leq 
  \inf_{T>0}
  \inf_{\varphi(\theta^*,T)}
    S_{T}(\varphi)
- \inf_{T>0}
  \inf_{\varphi(\theta^*,T)}
     \hat{S}_{T}(\varphi) \\
  \inf_{\theta\in\partial D}
  \inf_{T>0}
  \inf_{\varphi(\theta,T)}
  \hat{S}(\varphi)
- \inf_{\theta\in\partial D}
  \inf_{T>0}
  \inf_{\varphi(\theta,T)} S_{T}(\varphi)
\leq
  \inf_{T>0}
  \inf_{\varphi(\theta^\dagger,T)}
    \hat{S}_{T}(\varphi)
- \inf_{T>0}
  \inf_{\varphi(\theta^\dagger,T)}
    S_{T}(\varphi)
\end{align}
Similarly,
we can restrict our focus on finite $T$ in the exit time analysis (Lemma \ref{lemma:auxiliary_trajectories}).
Thus, for each $\theta$, we have the following finite $T^\dagger$ and $T^*$.
\begin{align}
  T^\dagger(\theta) := \underset{T>0}{\arg\!\inf} \inf_{\varphi(\theta,T)} S_{T}(\varphi) \\
  T^*(\theta) := \underset{T>0}{\arg\!\inf} \inf_{\varphi(\theta,T)} \hat{S}_{T}(\varphi),
\end{align}
and the followings hold
\begin{align}
  \inf_{T>0}
  \inf_{\varphi(\theta^*,T)} S_{T}(\varphi)
- \inf_{T>0}
  \inf_{\varphi(\theta^*,T)}
    \hat{S}_{T}(\varphi)
  \leq 
  \inf_{\varphi(\theta^*,T^*(\theta^*))}
          S_{T^*(\theta^*)}(\varphi)
- \inf_{\varphi(\theta^*,T^*(\theta^*))}
          \hat{S}_{T^*(\theta^*)}(\varphi) \\
  \inf_{T>0}
  \inf_{\varphi(\theta^\dagger,T)}
  \hat{S}_{T}(\varphi)
  - \inf_{T>0} \inf_{\varphi(\theta^\dagger,T)} S_{T}(\varphi)
  \leq
  \inf_{\varphi(\theta^\dagger,T^\dagger(\theta^\dagger))}
       \hat{S}_{T^\dagger(\theta^\dagger)}(\varphi)
 - \inf_{\varphi(\theta^\dagger,T^\dagger(\theta^\dagger))}
       S_{T^\dagger(\theta^\dagger)}(\varphi).
\end{align}
Similarly, since $L(\theta)$ and $C(\theta)$ are continuous,
for each $\theta$ and $T$, we have the followings 
\begin{align}
  \varphi^\dagger(\theta, T) := \underset{\varphi(\theta, T)}{\arg\!\inf} S_{T}(\varphi) \\
  \varphi^*(\theta, T) := \underset{\varphi(\theta, T)}{\arg\!\inf} \hat{S}_{T}(\varphi).
\end{align}
and we get
\begin{align}
  \inf_{\varphi(\theta^*,T^*(\theta^*))}
  S_{T^*(\theta^*)}(\varphi)
- \inf_{\varphi(\theta^*,T^*(\theta^*))} \hat{S}_{T^*}(\varphi)
  & \leq
  S_{T^*(\theta^*)}(\varphi^*(\theta^*,T^*(\theta^*)))
- \hat{S}_{T^*}(\varphi^*(\theta^*,T^*(\theta^*))) \\
  & \leq
  \frac{T^*(\theta^*)}{2}(K+M)^2
    \left(\lambda_\mathrm{min}^{-\frac{1}{2}} - \lambda_\mathrm{max}^{-\frac{1}{2}}\right) \\
\inf_{\varphi(\theta^\dagger,T^\dagger(\theta^\dagger))}
    \hat{S}_{T^\dagger(\theta^\dagger)}(\varphi)
- \inf_{\varphi(\theta^\dagger,T^\dagger(\theta^\dagger))}
    S_{T^\dagger(\theta^\dagger)}(\varphi) 
 & \leq
 \hat{S}_{T^\dagger(\theta^\dagger)}(\varphi^\dagger(\theta^\dagger,T^\dagger(\theta^\dagger)))
- S_{T^\dagger(\theta^\dagger)}(\varphi^\dagger(\theta^\dagger,T^\dagger(\theta^\dagger)))             \\
 & \leq \frac{T^\dagger(\theta^\dagger)}{2}
        (K+M)^2
        \left(\lambda_\mathrm{min}^{-\frac{1}{2}} - \lambda_\mathrm{max}^{-\frac{1}{2}}\right)
\end{align}
Thus we get

\begin{align}
  \left|V_0 -\lambda^{-\frac{1}{2}}\hat{V}_0\right|
  \leq \frac{\max\{T^\dagger(\theta^\dagger), T^*(\theta^*)\}}{2}(K+M)^2
  \left(\lambda_\mathrm{min}^{-\frac{1}{2}} - \lambda_\mathrm{max}^{-\frac{1}{2}}\right)
\end{align}
Defining $A:=\frac{\max\{T^\dagger(\theta^\dagger), T^*(\theta^*)\}}{2}(K+M)^2$ finishes the proof.

\end{proof}